\documentclass[10pt]{article} 

\usepackage[preprint]{rlj} 

\usepackage{amssymb}            
\usepackage{mathtools}          
\usepackage{mathrsfs}           
\usepackage{graphicx}           
\usepackage{subcaption}         
\usepackage[space]{grffile}     
\usepackage{url,booktabs}                
\usepackage{bm}
\allowdisplaybreaks

\title{Bandits for Efficient Experimentation: Adapting to Control Group, Preferences, and Context Drifts}

\setrunningtitle{Bandits with Control Group, Preferences, and Context Drifts}

\author{Udvas Das, Waris Radji, Debabrota Basu, Odalric-Ambrym Maillard}

\emails{udvas.das@inria.fr}

\affiliations{
Univ. Lille, Inria, CNRS,
Centrale Lille, UMR 9189 – CRIStAL 
F-59000 Lille, France
}

\contribution{
    In this work, we propose \framework~, the first MED-like algorithm that handles non-stationarity of contexts in every episode. We also provide instance-dependent upper bound on regret.
    }
    {
    No prior algorithm in the existing literature handles non-stationarity in contexts. 
    }

\contribution{
    We impose a safety constraint with respect to a base-policy (control-group). As we do not know the true mean parameter $\theta$ beforehand, this constraint is also stochastic in nature. We show \framework~can adapt to this constraint and the instance dependent regret upper bound involves a constraint-aware gap $\tilde{\Delta}$. We also prove with high probability, \framework~enjoys $\tilde{\bigO}(d)$ constraint violation
    }
    {
    Existing literature on non-stationary bandits and safe bandits are presently unwed. We affirmatively bridge these two literature.
    }

\contribution{
    We propose a novel constraint-aware Lagrangian dual penalised version of the G-optimal design to efficiently compute exact probabilities for each arm at every step. 
    }
    {
    To best of our knowledge, no prior work handles customised preferences per user and propose a preference divergent allocation strategy.
    }
\contribution{
    We conduct numerical experiments across synthetic datasets with varying types of context drifts. \framework~ outperforms the standard baseline algorithm significantly. 
    }
    {
    Our numerical results suggest that \framework~significantly outperforms conservative baselines that ignores the drift and preference structure.
    }


\summary{We consider a variant of the linear contextual stochastic multi-armed bandits,
where the learner must provide recommendations to a group of users, 
each having its personalised preference vector,  
and in the presence of context distributions that are drifting over time. 
Under practitioner-friendly assumptions, we reduce this setting to linear bandit with stationary mean but heteroskedastic and non-stationary noise. 
We further study the case when the learner must satisfy a \textit{baseline performance threshold}.
We introduce \texttt{Dri-MED}, an algorithm inspired from the linear version of the MED strategy, and carefully adapted to handle the non-stationary heteroskedastic noise. We prove upper bounds on instance-dependent regret and also expected number of constraint violations.
Our numerical results suggest that \texttt{Dri-MED} significantly outperforms conservative baselines that ignores the drift and preference structure.

}
\usepackage{array}
\usepackage{makecell}
\usepackage{float,xcolor,longtable}
\usepackage{tcolorbox}
\usepackage{mathtools}
\usepackage{amsmath, amssymb, natbib, graphicx, url}
\usepackage{algorithm}
\usepackage{algorithmic}
\usepackage{dsfont,enumitem}
\usepackage{footnote}
\usepackage{caption}
\usepackage{subcaption}
\usepackage{mwe}
\usepackage{cancel}
\usepackage{multirow}

\makesavenoteenv{tabular}
\makesavenoteenv{table}

\newcommand{\E}{\mathbb{E}}

\newcommand{\contextdist}[1]{\mathcal{C}_{#1}}

\newcommand \defn {\mathrel{\triangleq}}

\DeclareMathOperator*{\argmax}{arg\,max}
\DeclareMathOperator*{\argmin}{arg\,min}

\newcommand \expect {\mathop{\mbox{\ensuremath{\mathbb{E}}}}\nolimits}

\newcommand{\indicator}{\mathds{1}}

\newcommand{\reals}{\mathbb{R}}

\usepackage{amsthm}
\newtheorem{theorem}{Theorem}
\newtheorem{corollary}{Corollary}
\newtheorem{lemma}{Lemma}

\newtheorem{definition}{Definition}

\newtheorem{assumption}{Assumption}

\makeatletter
\newtheorem*{rep@theorem}{\rep@title}
\newcommand{\newreptheorem}[2]{%
	\newenvironment{rep#1}[1]{%
		\def\rep@title{\textbf{#2} \ref{##1}}%
		\begin{rep@theorem}}%
		{\end{rep@theorem}}}
\makeatother
\newreptheorem{theorem}{Theorem}
\newreptheorem{lemma}{Lemma}
\newreptheorem{proposition}{Proposition}
\newreptheorem{assumption}{Assumption}
\newreptheorem{corollary}{Corollary}

\DeclareRobustCommand{\bigO}{\text{\usefont{OMS}{cmsy}{m}{n}O}}
\allowdisplaybreaks%
\newif\ifdoublecol
\doublecolfalse

\usepackage[nohints]{minitoc}

\usepackage{mathrsfs}
\DeclareRobustCommand{\bigO}{%
  \text{\usefont{OMS}{cmsy}{m}{n}O}%
}

\usepackage{tikz}
\usetikzlibrary{automata}
\usetikzlibrary{bayesnet, arrows, calc, shapes, backgrounds,arrows,decorations.pathmorphing,fit,positioning}
\tikzset{
   container/.style = {rectangle, rounded corners, draw=yellow, dashed,
fit=#1, inner sep=6mm, node contents={}},
circle-label/.style = {circle, draw}
        }
\tikzset{box/.style={draw, diamond, thick, text centered, minimum height=0.5cm, minimum width=1cm, text width=0.9cm}}
\tikzset{line/.style={draw, thick, -latex'}}
\allowdisplaybreaks
\usepackage{pgfplots}

\newcommand{\framework}{{\color{blue}\ensuremath{\texttt{Dri-MED}}}}

\usepackage{layouts}

\def\bp{\mathbf{p}}

\def\btheta{{\boldsymbol\theta}}

\def\bnu{{\boldsymbol\nu}}

\def\bpi{{\boldsymbol\pi}}

\def\cA{\mathcal{A}}

\def\cC{\mathcal{C}}
\def\cD{\mathcal{D}}
\def\cE{\mathcal{E}}
\def\cF{\mathcal{F}}
\def\cG{\mathcal{G}}

\def\cJ{\mathcal{J}}

\def\cL{\mathcal{L}}

\def\cN{\mathcal{N}}

\def\cP{\mathcal{P}}

\def\cU{\mathcal{U}}
\def\cV{\mathcal{V}}

\def\E{\mathbb{E}}

\def\det{{\sf det}}

\usepackage{todonotes}

\newcommand{\Prob}{\mathbb{P}}

\newcommand{\Reg}{\mathrm{Reg}}

\newcommand{\vp}{\phi}

\newcommand{\Dpi}{\Delta_{\bpi_0}}
\newcommand{\Dpirs}{\underline{\Delta}_{\bpi_0}}

\newcommand{\inner}[2]{\langle #1,\, #2\rangle}
\newcommand{\thetahat}{\hat{\theta}}

\newcommand{\suml}{\sum_{\ell=1}^L}
\newcommand{\sumh}{\sum_{h=1}^H}
\newcommand{\context}{\bf c_\ell}
\newcommand{\gram}[2]{V_{#1,#2}^{\gamma_{\mathrm{decay}}}}
\newcommand{\graml}[1]{V_{#1}^{\gamma_{\mathrm{decay}}}}
\newcommand{\discount}{\gamma_{\rm decay}}
\newcommand{\truegap}{\Delta_{A_{h,\ell},h,\ell}}
\newcommand{\truegaprs}{\underline{\Delta}_{A_{h,\ell},h,\ell}}
\newcommand{\estgap}{\hat{\Delta}_{A_{h,\ell},h,\ell}}
\newcommand{\estgaprs}{\hat{\underline{\Delta}}_{A_{h,\ell},h,\ell}}
\newcommand{\bestarm}{a_{h,\ell}^*}
\newcommand{\feature}[1]{\phi_{#1,h,\ell}}

\newcommand{\bluecheck}{}%
\DeclareRobustCommand{\bluecheck}{%
  \tikz\fill[scale=0.4, color=blue]
  (0,.35) -- (.25,0) -- (1,.7) -- (.25,.15) -- cycle;%
}\newcommand{\xmark}{%
\tikz[scale=0.23, color=red] {
    \draw[line width=0.7,line cap=round] (0,0) to [bend left=6] (1,1);
    \draw[line width=0.7,line cap=round] (0.2,0.95) to [bend right=3] (0.8,0.05);
}}
\begin{document}

\maketitle  

\begin{abstract}
We consider a variant of the linear contextual stochastic multi-armed bandits,
where the learner must provide recommendations to a group of users, 
each having its personalised preference vector,  
and in the presence of context distributions that are drifting over time. 
Under practitioner-friendly assumptions, we reduce this setting to linear bandit with stationary mean but heteroskedastic and non-stationary noise. 
We further study the case when the learner must ensure the mean reward of each decision must exceed that of a baseline strategy $\bpi_0$ at each decision step.
We introduce \texttt{Dri-MED}, an algorithm inspired from the linear version of the MED strategy, and carefully adapted to handle the non-stationary heteroskedastic noise.
We show that the instance-dependent regret scales as $\tilde{\bigO}\left(\frac{\kappa}{\tilde{\Delta}}d^2(\log(T)\right)$, where $\tilde{\Delta}$ is the constraint-aware sub-optimality gap subject to policy $\pi_0$, with variance-aware multiplicative term $\kappa$ that we carefully handle using heteroskedastic regression. We further show \texttt{Dri-MED} enjoys $\tilde{\bigO}(d)$ expected constraint violations. 
Our numerical results suggest that \texttt{Dri-MED} significantly outperforms conservative baselines that ignores the drift and preference structure.
\end{abstract}

\vspace*{-0.5em}\section{Introduction}



Design of Experiments (DoE) has been an integral component of agricultural research, clinical trials, and natural sciences~\citep{neyman,aylmer1926arrangement,diggle2011,hoshmand2018design}. Suppose a farmer testing whether a new fertilizer, seed variety, or irrigation protocol
outperforms the current standard cannot simply apply the new treatment to one plot and
draw conclusions. Numerous factors such as soil fertility gradients, microclimatic variation, pest pressure, might confound any naive comparison
\citep{gomez1984statistical, tanner2023experimenting}. This fundamental challenge DoE aims to address lies in separating the signal of a treatment effect from the noise
of environmental heterogeneity. Pioneered by \citep{aylmer1926arrangement} at the Rothamsted
Experimental Station, DoE provides a principled statistical framework for planning
agricultural field trials so that \textit{valid, reliable, and efficient} conclusions can be drawn
from inherently noisy observations \citep{rangaswamy1995text}. 

The first and most fundamental objective is \emph{Validity}, i.e., ensuring that an observed difference between a treated plot and a control plot can be causally attributed to the
treatment itself, and not to pre-existing differences between plots. Thus, randomized experiment is at the core of sound statistical validation~\citep{banerjee2020theory,raccuglia2016machine,kohavi2015online} of a experimental design. Treatments must be assigned to experimental units by a chance
mechanism, also referred to as an allocation distribution.

\emph{Reliability:} Validity alone is insufficient in practice.
Even a perfectly randomized experiment may fail to be \emph{reliable} if, in the process
of exploring new treatments, some field sites systematically receive practices that perform
far worse than the current default.
In agriculture, this is not merely a statistical concern but an operational and ethical one.
A farmer whose plot is assigned a poorly performing treatment suffers an economically
damaging yield loss, potentially threatening food security or livelihood. Thus, it becomes imperative to enforce a \textit{performance/safety constraint}~\citep{pacchiano2021stochastic,das2024learning,amani2019linear}, to ensure reliability.
Safety constraints are typically imposed by performing a baseline action/policy over a control group in each site.

\emph{Efficiency:} Finally, as real-world experimentation is costly, we want to minimize the suboptimal decisions over time and avoid unsafe decisions underperforming that of default baseline. 

This common, albeit still idealized situation naturally fits the sequential decision-making~\citep{sutton,lattimore_szepesvari_2020} framework, where one observes the impact of an action or intervention by deploying it in an agricultural site and then adapts the future decisions accordingly. Now, suppose we observe a fixed set of farmers across different season. The change in environmental factors can be modelled as drifts in the decision making contexts over multiple interaction episodes. 
In addition, a critical but frequently overlooked dimension of agricultural experimentation is that different farmers evaluate the outcomes of a treatment differently~\citep{aylmer1926arrangement,rangaswamy1995text}.
A smallholder farmer may weigh yield stability above raw yield maximization, or a farmer in a water-scarce region may prioritize water-use efficiency over all other metrics. Hence, it calls for a design diverse enough that can adapt to these customised preferences.

In this work, we study this problem of designing valid, reliable, and efficient sequential experiments for a finite set of farmers experimenting across multiple agricultural sites with control groups and preferences, and across multiple episodes with plausible context drifts to find out optimal interventions over time. Though we elucidate our problem motivation with the agricultural example, the same problem reoccurs across domains, such as clinical trials for drugs and treatments~\citep{reda2020machine,munro2021safety}, aircraft design optimisation~\citep{mavrotas2009effective}, recommender systems~\citep{feijer2025calibrated,saha2026one} to name a few.

\subsection{Problem Setting}\label{sec:formulation}
We formalise this problem as a variant of the multi-armed bandit problem with a finite set of arms $\cA$, from which a learner must recommend actions to be played to a group of $H$ users (or experimental sites), in $L$ successive rounds called \textit{episodes}. 

\textbf{Contexts.} The rewards depend on a context that might evolve over episodes. Precisely, within an episode $\ell\in[L]$, the contexts $C_{h,\ell}$ are generated from a context distribution $\contextdist{\ell}$ for all users $h\in[H]$. But the context distribution is non-stationary, i.e. it varies over the episodes $\ell$.
Then, the learner must recommend one action for each user $h\in[H]$. 
Upon recommending the actions $A_{h,\ell}$,
the learner receives  $M$-dimensional feedbacks $Y_{h,\ell}$, where
$Y_{h,\ell}$ is a random vector of dimension $M$ generated by an unknown distribution, with a specific linear structure (Assumption~\ref{ass:linear}).

\textbf{Preferences.} Each user $h$ further has a preference $\bp_h\in\mathbb{R}^M$ over the vectorial feedback. This further yields a reward $r_{h,\ell}(a)\defn Y_{h, \ell}^\top {\bf p}_h$ for user $h$ upon receiving an action $A_{h,\ell}=a$. 
The goal of the learner is to suggest a sequence of actions $\{A_{h,\ell}\}$ to the users over multiple episodes that maximises the expected cumulative reward, i.e., $\expect\left[\suml\sumh r_{h,\ell}(A_{h,\ell})\right]$.

\textbf{Baseline policy.} To simulate the control group and action, a mixed baseline policy $\pi_0\in\cP(\cA)$ is given to the learner. Given $\pi_0$ and a tolerance level $(1-\epsilon)\in [0,1]$, the learner must ensure to accumulate at least $(1-\epsilon)\mathbb{E}_{\pi_0}[r_{h,\ell}]$  reward for each user in each round. This additional constraint forces the learner to recommend policies that are at least as good or better than the baseline policy applied on the control group.

\noindent\textbf{Performance metrics:} In this setting, performance of an algorithm is measured with two metrics: 
cumulative regret and expected cumulative number of constraint violations. 
Lower values of both the metrics indicate better performance for an algorithm.

\textbf{Regret.} The cumulative regret is defined as
\begin{equation}\label{eq:regret}
    \Reg = \mathbb{E}\!\left[
        \sum_{\ell=1}^L \sum_{h=1}^H
        \Delta_{A_{h,\ell},h,\ell}
    \right]
    = 
        \expect\left[\sum_{\ell=1}^L \sum_{h=1}^H
        \Bigl({\bf m}^\star_{h,\ell}-{\bf m}_{h,\ell}(A_{h,\ell})\Bigr)\right].
\end{equation}
Here, the sub-optimality gaps $\Delta_{a,h,\ell}\defn {\bf m}^\star_{h,\ell}-{\bf m}_{h,\ell}(a)$, where ${\bf m}_{h,\ell}(a)$ denotes the mean of reward $r_{h,\ell}(a)$ and ${\bf m}^\star_{h,\ell}\defn \max\limits_{a}{\bf m}_{h,\ell}(a)$. Also, $\Delta_0$ denotes the minimum sub-optimality gap from $\pi_0$.

\textbf{Constraint Violation.} On the other hand, the safety is enforced by measuring the cumulated constraint violation with respect to a stationary baseline policy $\pi_0\in\cP(\cA)$,
\begin{equation}\label{eq:constraints}
\text{Violation}({\pi_0}) = \mathbb{E}\!\left[
        \sum_{\ell=1}^L \sum_{h=1}^H \indicator\!\Big\{{\bf m}_{h,\ell}(A_{h,\ell})<(1-\epsilon)\mathbb{E}_{a\sim \pi_0}[{\bf m}_{h,\ell}(a)]\Big\}\right]\,.
\end{equation}
We observe that while the works on safe linear bandits need to assume explicitly existence of a safe arm satisfying the constraint~\citep{pacchiano2021stochastic}, we do not need such assumption. 
This is due to existence of a mixed baseline policy and the constraint defined by it. Thus, $\forall h\in[H], \ell\in[L]$, there exists at leasts one action $a\in \mathcal{A}$ satisfying ${\bf m}_{a,h,\ell}\geq(1-\epsilon)\mathbb{E}_{a\sim \pi_0}[{\bf m}_{h,\ell}(a)]$ for any $\epsilon<1$.

Now, we elaborate the structural assumptions that we adhere to for the rest of the paper.

\begin{assumption}[Gaussian  feedback signal with context-independent mean]\label{ass:cont_indep}
 The feedback signal is Gaussian with mean assumed context-independent. Formally,  for any $h\in[H]$ and $\ell\in[L]$, $Y_{h,\ell}  \mid (A_{h,\ell} = a,C_{h,\ell}=c) \sim \cN(m_a,\Sigma_{a,c})$, where
  $m_a\in\mathbb{R}^{M}$ is \textbf{constant} in $c$ and 
   $\Sigma_{a,c} \in \mathbb{R}^{M \times M}$ is \textbf{a positive semi-definite matrix}
  that varies with both arm $a$ and context $c$. In particular\footnote{
  This implies independence between sites, since the distribution of $Y_{h,\ell}$ does not depend on $(A_{h',\ell'},C_{h',\ell}')_{(h',\ell')\neq (h,\ell)}$  
  }
\begin{equation}\label{eq:mean_indep}
    \mathbb{E}[Y_{h,\ell} \mid A_{h,\ell} = a,C_{h,\ell}=c] = m_a\,,\quad
 	\mathrm{Cov}[Y_{h,\ell} \mid A_{h,\ell} = a,\, C_{h,\ell}=c]
 	= \Sigma_{a, c}.
\end{equation}
\end{assumption}
In particular, under Assumption~\ref{ass:cont_indep},
the mean rewards are $\ell$-independent, i.e., 
$\forall \ell, {\bf m}_{h,\ell}(a)= m_a\bp_h$.
Hence, the optimal mean ${\bf m}^\star_\ell$ and action $a^\star_{h}$ are also episode independent and user dependent. Assumption~\ref{ass:cont_indep} is prevalent in many real-life settings. For example, in clinical trials mean response for a given dose remains stationary but variance in patients' responses is heteroskedastic~\citep{legedza2001heterogeneity}. In case of recommender systems having arms as content categories (sports, politics), average engagement (click, watch-time) or reward for a category is stable over weeks but per-user variance can vary~\citep{saha2026one,feijer2025calibrated}.
\begin{assumption}[Linear structure]\label{ass:linear}
	The mean signal satisfies $m_a= \Phi(a)\theta$,  for some 
	unknown parameter $\theta \in \mathbb{R}^d$ with $\|\theta\|_2 \le S_*$,
	and a \textbf{known} feature matrix $\Phi(a) \in \mathbb{R}^{M \times d}$.
	In particular,   the normalized reward    $\tilde r_{h,\ell}(a)=r_{h,\ell}(a)/\sigma_{a,h,\ell}$ satisfies
	$	\tilde r_{h,\ell}(a) \mid (C_{h,\ell}=c) \sim \cN(\theta^\top \phi_{a,h,\ell},1)$, where 
	\begin{equation}\label{eq:scalar_var}
		\phi_{a,h,\ell}=\frac{1}{\sigma_{a,h,\ell}} \Phi(a)^{\top}\bp_h \in \mathbb{R}^d, \quad
		\sigma^2_{a,h,\ell}
		= \bp_h^\top \Sigma_{a,C_{h,\ell}}\,\bp_h>0.
	\end{equation}\vspace*{-1em}
\end{assumption}
The variance $\sigma^2_{a,h,\ell}$ in Assumption~\ref{ass:linear}, is heteroskedastic in three senses simultaneously:
it varies across actions $a$, users $h$ (through $\bp_h$),
and episodes $\ell$ (through $C_{h,\ell} \sim \contextdist{\ell}$). We further introduce three quantities: (a) maximum variance per step $\bar{\sigma}^2_{h,\ell} = \max\limits_a \sigma_{a,h,\ell}^2$, (b) condition number of covariances over the episodes $\kappa \defn \frac{\max_{a,h,\ell}\sigma_{a,h,\ell}^2}{\min_{a,h,\ell}\sigma_{a,h,\ell}^2} \leq \infty$, and (c) discounted cumulative variance $\Sigma_{LH} \defn \suml{(\discount)}^{L-\ell} \sumh \bar{\sigma}^2_{h,\ell}$ for some discount factor $\gamma_{decay} \in (0,1]$.


\noindent\textbf{Discussion.}
Context-independence in Assumption~\ref{ass:cont_indep} is naturally satisfied in applications where one base action, called the control, is available at each site location $h$, and the practitioner considers as feedback signal the difference of observed effects of action $a$ and of the control. While we assume Gaussian noise, our theory and algorithm design extend to the general sub-Gaussian noise.

\subsection{Outline and Contributions} 

We consider the problem of linear contextual bandits (LinCB) under non-stationary or drifting context distribution for $H$ users encountered in each of $L$ episodes, where each user comes with a customised preference feedback. We investigate the following questions:
\begin{tcolorbox}[top=1pt,bottom=1pt,left=2pt,right=2pt]
1. Can we design an algorithm to solve this problem of safe LinCB under context drifts and customized preferential feedback, and derive an instance-dependent upper bound on its regret?

2. If we have a performance constraint subject to a control-group, or equivalently a safety constraint with respect to a base-policy $\bpi_0$, can we adapt to it, i.e., can we derive an upper bound on expected number of constraint violations?

3. Can we optimally design sampling probabilities at each step that will be simultaneously constraint-aware and preference divergent (sufficient exploration in all customised preference directions)?
\end{tcolorbox}


We aim to design a randomized strategy that efficiently computes the closed form probabilities of selecting each arm, since they prove to be beneficial in downstream tasks such as offline evaluation via inverse propensity score. In this work, we affirmatively answer the above questions and state them as main contributions below:

\textbf{1. Context Drift and Customised Preference.} In this work, we propose \framework~, the first MED-like algorithm that handles non-stationarity of contexts in every episode. We also prove an upper bound on the instance dependent regret of order $\tilde{\bigO}\left(  \frac{\kappa d^2}{\tilde{\Delta}}\log^2(LH)\right)$. This non-trivial condition number captures sensitivity due to strength of drift in contexts across episodes. 

\textbf{2. Safety Constraint.} In our setting, we impose a safety constraint with respect to a base-policy (control-group). As we do not know the true mean parameter $\theta$ beforehand, this constraint is also stochastic in nature. We show \framework~can adapt to this constraint and the instance dependent regret upper bound involves a constraint-aware gap $\tilde{\Delta}$. We also prove with high probability, \framework~enjoys $\tilde{\bigO}(d)$ constraint violation.

\textbf{3. Constraint-aware and Preference Divergent Design.} We propose a novel constraint-aware Lagrangian dual penalised version of the G-optimal design~\citep{balagopalan2024minimum} to efficiently compute exact probabilities for each arm at every step. To best of our knowledge, no prior work handles customised preferences per user and propose a preference divergent allocation strategy. We also verify this claim in our experimental analysis, as \framework~ assigns more allocation on the best arm consistently over baselines.

\textbf{4. Empirical Performance Gain.} We conduct numerical experiments across synthetic datasets with varying types of context drifts, namely abrupt, periodical, gradual or no drift. We observe \framework~ outperforms the standard baseline algorithm significantly. Additionally, we propose an IMED~\citep{baudry2023fast,honda2015non} version of \framework~ for empirical comparison, namely \texttt{Dri-IMED}~(see Section \ref{supp:dri-imed}). The performance of \texttt{Dri-IMED} proves to be competitive consistently with \framework~across all types of drift. \vspace*{-0.8em}

\begin{table}[t!]
\resizebox{\textwidth}{!}{
\begin{tabular}{c c c c c}
\toprule
 Algorithms & Instance Dependent Regret & Minimax Regret & Context Drift & Safety Constraint\\ 
\midrule
 OFUL~\citep{abbasi2011improved} & $\tilde{\bigO}\left( \frac{d^2}{\Delta_{\min}}\log^3 (T)\right)$ &$ \tilde{\bigO}\left( d\sqrt{T}\right)$ & \xmark & \xmark\\
 SpannerIGW~\citep{zhu2022contextual} & $\Omega \left(\frac{1}{\Delta_{\min}}\sqrt{T}\right)$& ${\bigO}\left( \sqrt{dT\log K}\right)$ & \xmark & \xmark\\
 OPLB~\citep{pacchiano2021stochastic}& \xmark & $\tilde{\bigO}\left( \frac{d\sqrt{T}}{\Delta_{\min}}\right)$ & \xmark & \bluecheck\\
 SOLID~\citep{tirinzoni2020asymptotically} & \xmark & $\tilde{\bigO}\left( \sqrt{dT} \right)$ & \xmark & \xmark\\
 LinIMED~\citep{bian2024indexed} & \xmark & $\tilde{\bigO}\left( d\sqrt{T}\right)$ & \xmark & \xmark\\
 LinMED~\citep{balagopalan2024minimum} & $\tilde{\bigO}\left( \frac{d^2}{\Delta_{\min}}\log^2 (T)\right)$ & $\tilde{\bigO}\left( d\sqrt{T}\right)$ & \xmark &  \xmark\\
 \framework~ (This work) &  $\tilde{\bigO}\left( \frac{\kappa d^2}{\tilde{\Delta}}\log^2(T)\right)$ & $\tilde{\bigO}\left( d\sqrt{\kappa T}\right)$ & \bluecheck &\bluecheck \\
\bottomrule
\end{tabular}}\vspace*{-.5em}
\caption{Comparison of Regret Upper Bounds of \framework~against SOTA methods. Here, $T=LH$ and $\tilde{\Delta} = \min\{\Delta_{\min},\Delta_0\}$, and $\kappa \defn \frac{\max_{a,h,\ell}\sigma_{a,h,\ell}^2}{\min_{a,h,\ell}\sigma_{a,h,\ell}^2}$.}\label{tab:comparison}\vspace*{-1.5em}
\end{table}

\section{Related Work} 
This problem lies at the crossroad of different settings of bandit literature. While the linear parametrisation of rewards allows us to leverage the rich literature of linear contextual bandits, the context drifts over episodes are related to non-stationary and heteroskedastic linear bandits. Finally, the baseline policy requires us to bring techniques from bandits with safety constraints and extend them further. Here, we summarise the relevant literature.

Since the seminal OFUL (homoskedastic) strategy is proposed by \cite{abbasi2011improved}, the literature in linear multi-armed bandits has considerably expanded over the last decade \citep{krause2011contextual,valko2014spectral,lattimore2017end,chowdhury2017kernelized,abeille2017exploration,durand2018streaming},
culminating in provably instance-dependent optimal and efficient strategies LinIMED~\citep{bian2024indexed} and LinMED~\citep{balagopalan2024minimum}, respectively inspired from the IMED~\citep{honda2015non,baudry2023fast} and MED (aka Maillard sampling)~\citep{honda2011asymptotically,qin2023kullback,qin2025achieving}  for unstructured bandits. 

But the classical linear contextual bandits do not consider non-stationarity of contexts and also assume the variance across contexts and rounds to be static.
To mitigate non-stationarity, \cite{russac2019weighted}  and \cite{faury2021regret} propose time-discounted regression estimates in (generalized) linear bandits under bounded variation budget. 

Heteroskedasticity in linear bandits has been studied in \citep{he2025variance}, establishing variance-aware lower performance bounds. 
\cite{kirschner2018information,kim2026jointly} assume known local variance to provide worst-case regret bounds in $\mathcal{O}((d\log(T)+\log(1/\delta))\sqrt{T})$, while \cite{bregere2019target} extends to the unknown heteroskedastic variance. 

Batch bandits have received increasing attention over the past decade~\citep{perchet2016batched,gao2019batched,jin2021almost}, where the goal is to perform similarly to the pure sequential bandit problem with a minimal number of adaptively chosen batch sizes. Closer to this setting, \cite{zhang2020inference} study when constant batch sizes are given as a constraint. This problem is also reminiscent of episodic reinforcement learning and combinatorial bandits, although with simpler structure.

Regarding the performance constraint, safety with respect to constraint violation has been studied in \cite{amani2019linear}, \cite{pacchiano2021stochastic}. 
Finally, we mention the somewhat related notion of satisficing objective \cite{feng2025satisficing}, as well as the study of bandits under linear constraints from \cite{das2024learning}, for the different pure exploration objective.
\begin{algorithm}[t!]
\caption{LinMED~\citep{balagopalan2024minimum}}
\label{alg:linmed}
\begin{algorithmic}[1]
    \STATE Initialise $\hat{\theta}_0 = 0$, $V_0 = \lambda I$
    \FOR{$t = 1, 2, \ldots$}
        \STATE Observe arm set $\mathcal{A}_t$; compute $\hat{a}_t = \arg\max_{a'\in\mathcal{A}_t}\langle\hat{\theta}_{t-1}, a'\rangle$
        \STATE Compute gaps $\hat{\Delta}_{a,t} \leftarrow \langle\hat{\theta}_{t-1},\, \hat{a}_t - a\rangle$ for all $a \in \mathcal{A}_t$
        \STATE Compute weights $f_t(a)$ via for all $a \in \mathcal{A}_t$
        \STATE $p_t'(a) \leftarrow g(q_t^{\rm opt}(a))f_t(a)\,/\!\sum_{b\in\mathcal{A}_t} g(q_t^{\rm opt}(b))f_t(b)$
        \STATE $\mathcal{B}_t \leftarrow \{a \in \mathcal{A}_t : \|a\|^2_{V_{t-1}^{-1}} > 1\}$
        \IF{$|\mathcal{B}_t| > 0$}
            \STATE $p_t(a) \leftarrow \tfrac{1}{2}p_t'(a) + \tfrac{1}{2}\,\indicator\{a = B_t\}$, \ $B_t \in \mathcal{B}_t$ arbitrary
        \ELSE
            \STATE $p_t(a) \leftarrow p_t'(a)$
        \ENDIF
        \STATE Sample $A_t \sim p_t$; observe $Y_t$; update $V_t \leftarrow V_{t-1} + A_tA_t^\top$, \ $\hat{\theta}_t \leftarrow V_t^{-1}\sum_{s=1}^t A_s Y_s$
    \ENDFOR
\end{algorithmic}
\end{algorithm}

\textbf{A Primer on MED and LinMED.} 
MED 
(Minimum Empirical Divergence) is a family of randomised bandit algorithms based on the probability matching philosophy~\citep{honda2011asymptotically}. MED-type algorithms pull each arm according to a probability dictating how likely it might be the optimal arm.
Though originally proposed for bounded rewards and unstructured bandits, it has been extended to sub-Gaussian rewards, and is often referred to as Maillard sampling~\citep{bian2024indexed}.
\citet{balagopalan2024minimum} extended the MED strategy to the linear bandits and proposed LinMED. At each round $t$, it identifies the 
empirical best arm $\hat{a}_t = \argmax_{a \in \mathcal{A}_t} \langle \hat{\theta}_{t-1}, a \rangle$ 
and computes exponential weights $f_t(a) = \exp\!\left(-\dfrac{\hat{\Delta}_{a,t}^2}{\beta_{t-1}(\delta_{t-1})\,\|\hat{a}_t - a\|^2_{V_{t-1}^{-1}}}\right)$, where $\hat{\Delta}_{a,t} \coloneqq \langle \hat{\theta}_{t-1},\, \hat{a}_t - a \rangle$ 
is the estimated sub-optimality gap and $\beta_t(\delta_t)$ is a confidence scaling factor. These weights are used to compute an approximate $G$-optimal design 
$q_t^{\mathrm{opt}}$ over $\mathcal{A}_t$, which is then blended with a mass on the 
empirical best arm and a uniform component to form the final sampling distribution $p_t$. 
A saturation check $\mathcal{B}_t = \{a \in \mathcal{A}_t : \|a\|^2_{V_{t-1}^{-1}} > 1\}$ 
forces uniform exploration of under-sampled arms. The pseudocode is given in 
Algorithm~\ref{alg:linmed}. For further details, we refer to~\citep{balagopalan2024minimum}.

\vspace{-.5em}
\section{\framework: Drift Adaptive Minimum Empirical Divergence Algorithm}\label{sec:algorithm}\vspace{-.5em}

Now we are ready to propose our algorithm \framework: drift adaptive minimum empirical divergence for safe linear contextual bandits under customised preference feedback. 
\framework~ extends the MED-type algorithms in this setting. We choose MED-type algorithms for two reasons: (a) it yields tight problem-dependent regret bound for both independent arms~\citep{baudry2023fast} and linear contextual settings~\citep{balagopalan2024minimum}, (b) it allows an experimental design over the actions~\citep{zhu2022contextual,balagopalan2024minimum}, i.e., it controls a non-zero probability of pulling each action for every user that leads to enough data collection for each of the actions and facilitates downstream evaluation of goodness of different actions against the baseline policy.
The main challenges that \framework~resolve in addition to the LinMED algorithm~\citep{balagopalan2024minimum} are adapting to heteroskedastic noise across users with preferences, context drifts across episodes, and constraint violations against the baseline policies. We resolve them in three phases.

\setlength{\textfloatsep}{10pt}
\begin{algorithm}[t!]
\caption{\framework: \textbf{Dri}ft adaptive \textbf{M}inimum \textbf{E}mpirical \textbf{D}ivergence}
\label{alg:episodic-pref-imed-full}

\begin{algorithmic}[1]
\REQUIRE Baseline parameters $\pi_0, \epsilon$, regularization parameter $\lambda$, $S_*$
\REQUIRE Discount over drifts $\gamma_{\text{decay}} \in [0,1]$ and confidence levels $\{\delta_\ell\}$
\REQUIRE {Initial Lagrangian multiplier} $\nu_0 > 0$, {Lagrangian step size schedule} $\{\eta_{h,\ell}=(\ell h)^{-\frac12}\}$

\STATE Initialize $\hat{\theta}_{0,0} = 0$, $V_0 = \lambda I_d$, $\bar{C}_0 = 0$, and $N_0(a) = 0 \; \forall a$

\FOR{episode $\ell = 1$ to $L$}
    \STATE Initialize an episode with $V_{0,\ell}^{\gamma_{\text{decay}}} = \discount^{L-\ell} V_{\ell-1}^{\discount} $, $\hat{\theta}_{0,\ell} = \frac{1}{\gamma_{\text{decay}}}\hat{\theta}_{\ell-1}$, and $\nu_{h,\ell} = \nu_{\ell-1}$
    
    \textcolor{blue}{\textbf{Phase 1: Context observation \& drift detection}}
    \STATE Observe contexts $\{C_{h,\ell}\}_{h=1}^H \sim \contextdist{\ell}$ and preferences $\{\bp_{h}\}_{h=1}^H$
    \STATE For all $a \in \mathcal{A}$, compute per-user features $\phi_{a,h,\ell} = \frac{1}{\sigma_{a,h,\ell}}\Phi(a)^\top \bp_{h}$ and mean feature $\bar{\phi}_{a,\ell} = \frac{1}{H} \sumh \phi_{a,h,\ell}$
        
    \textcolor{blue}{\textbf{Phase 2: Episode-specific baseline \& constraint}}
    \STATE Confidence width on baseline $\sqrt{\beta_{\ell}(\delta_\ell)} = \sqrt{\log\frac{\det V_{\ell-1}^{\gamma_{\text{decay}}}}{\det V_0} + 2\log\frac{1}{\delta_\ell}} + \sqrt{\lambda}S_*$
    \STATE Compute the optimistic baseline $\hat{\mu}_{0,\ell}^+$ with Equation~\eqref{eq:optimistic_baseline}
    
    \STATE Set the constraint threshold $\tau_\ell = (1-\epsilon)\hat{\mu}_{0,\ell}^+$\\
    \textcolor{blue}{\textbf{Phase 3: Within-episode decisions}}
    \FOR{user $h = 1$ to $H$}
        
        \STATE Play an action $A_{h,\ell} \gets$  Apply Algorithm~\ref{alg:med_step}
        \STATE Observe $Y_{h,\ell}$, $\bp_h$, and $\tilde{r}_{h,\ell} = \frac{1}{\sigma_{A_{h,\ell},h,\ell}}\bp_h^\top Y_{h,\ell}$\\
        \textcolor{blue}{\textbf{Update Lagrangian multiplier, Gram matrix, and parameter estimate}}
        \STATE Constraint violation indicator: $\xi_{h,\ell} = \indicator\left\{\langle \hat{\theta}_{h,\ell}, \phi_{A_{h,\ell},h,\ell} \rangle < \tau_\ell\right\}$
        \STATE Update the Lagrangian multiplier 
        $\nu_{h,\ell} \leftarrow \max\left\{0, \nu_{h,\ell} + \eta_{h,\ell}\left(\xi_{h,\ell} - \frac{\epsilon}{1-\epsilon}\right)\right\}$
        
        
        \STATE Update $N_{h+1,\ell}(A_{h,\ell}) = N_{h,\ell}(A_{h,\ell}) + 1$, $V_{h,\ell}^{\gamma_{\text{decay}}} = V_{h-1,\ell}^{\gamma_{\text{decay}}} + \phi_{A_{h,\ell},h,\ell}\phi_{A_{h,\ell},h,\ell}^\top$
        \STATE Compute $\hat{\theta}_{h,\ell} = (V_{h,\ell}^{\gamma_{\text{decay}}})^{-1} \sum_{s=1}^h  {\phi_{A_{s,\ell},s,\ell} r_{A_{s,\ell}}}$
    \ENDFOR
        
    \STATE Store episodic statistics $V_\ell^{\gamma_{\text{decay}}} \gets V_{H,\ell}^{\gamma_{\text{decay}}}, \hat{\theta}_\ell \gets \hat{\theta}_{H,\ell}, \bar{C}_\ell, \mu_{0,\ell}, \nu_{h,\ell}$
    
\ENDFOR

\end{algorithmic}
\end{algorithm}

\begin{algorithm}[t!]
    \caption{Baseline Adaptive Minimum Empirical Divergence (MED) Step}
    \label{alg:med_step}

\begin{algorithmic}[1]
    \REQUIRE Episode $\ell \in [L]$, user id $h\in[H]$, user's context $\left\{\phi_{a,h,\ell}\right\}_{a=1}^{|\cA|}$, $\thetahat_{h,\ell}, V_{h,\ell}^{\gamma_{\rm decay}},\lambda$.
    \REQUIRE Constraint threshold $\tau_\ell$, Lagrangian multiplier $\nu_{h,\ell}, \alpha_{\rm emp}, \alpha_{\rm opt}$.
    \STATE  \textcolor{blue}{\textbf{Compute confidence radius:}} ${\beta_{h,\ell}(\alpha_{h,\ell})}^{1/2}  =  \sqrt{\log\frac{\det V_{h,\ell}^{\gamma_{\rm decay}}}{\det V_0} + 2\log\frac{1}{\alpha_{h,\ell}}} + \sqrt{\lambda} S_*$
    
    \STATE For each arm $a\in\cA$: $\text{LCB}(a,h,\ell) = \langle \hat{\theta}_{h,\ell}, \phi_{a,h,\ell} \rangle - {\beta_{h,\ell}(\alpha_{h,\ell})}^{1/2} \|\phi_{a,h,\ell}\|_{(V_{h,\ell}^{\gamma_{\rm decay}})^{-1}}$
    \STATE Empirical best action: $\hat{a}_{h,\ell} = \argmax_{a\in\cA} \langle \thetahat_{h,\ell}, \phi_{a,h,\ell} \rangle$
    \STATE Empirical gaps: $\hat{\Delta}_{a,h,\ell} = \langle\thetahat_{h,\ell},\,\phi_{\hat{a}_{h,\ell},h,\ell} - \phi_{a,h,\ell}\rangle$ 
    \STATE \textcolor{blue}{\textbf{Compute Constraint-aware pulling probabilities:}} For every $ a \in \mathcal{A}$, compute $\tilde{f}_h(a)$ with Equation~\eqref{eq:design}.
        \STATE Rescale arms:\quad
               $\mathcal{A}^{(h)} = \bigl\{\sqrt{\tilde{f}_h(a)}\cdot\phi_{a,h,\ell}
               \;\big|\; a\in\mathcal{A}\bigr\}$
        \STATE $q_h^{\mathrm{opt}} = \textcolor{red}{\mathrm{ApproxDesign}}\left(\mathcal{A}^{(h)}\right)$ $\gets $ Algorithm~\ref{alg:approx_design}
        \STATE
        $q_h(a) = \alpha_{\mathrm{opt}}\,q_h^{\mathrm{opt}}(a)
                + \alpha_{\mathrm{emp}}\,\mathbf{1}[a=\hat{a}_{h,\ell}]
                + (1-\alpha_{\mathrm{opt}}-\alpha_{\mathrm{emp}})\,\tfrac{1}{|\mathcal{A}|}$
        \STATE Compute final design probabilities:
        $p_h'(a) = {q_h(a)\,\tilde{f}_h(a)}/{\displaystyle\sum_{b}q_h(b)\,\tilde{f}_h(b)}$

        \STATE Check saturation event for each arms: $B_{h,\ell} = \bigl\{a:\|\phi_{a,h,\ell}\|^{2}_{(V_{h,\ell}^{\gamma_{\rm decay}})^{-1}}>1\bigr\}$
        \STATE
        $p_h(a) =
        \begin{cases}
            \tfrac{1}{2}\,p_h'(a) + \tfrac{1}{2}\,\mathbf{1}[a\in B_{h,\ell}]
              & \text{if } B_{h,\ell} \neq \emptyset \\
            p_h'(a) & \text{otherwise}
        \end{cases}$

        \STATE Sample $A_{h,\ell}\sim p_h$
\end{algorithmic}
\end{algorithm}


\textit{Phase 1:} At the beginning of $\ell$-th episode, we observe the contexts $\{C_{h,\ell}\}_{h=1}^H \sim \contextdist{\ell}$ and preferences $\{\bp_h\}_{h=1}^H$ for all the $H$ users. We compute per-user feature vectors $\Phi(a)^\top \bp_h$ and then normalise with their corresponding $\sigma_{a,h,\ell}$ to tackle heteroskedasticity.

\textit{Phase 2:} Since the true reward mean of the baseline policy $\bpi_0$, i.e., $\mu_0 \defn \expect_{a\sim \pi_0}\left[\frac{1}{H}\sumh \tilde{r}_h(a)\right]$ is unknown in the beginning, we construct an optimistic estimate of it at the beginning of each episode $\ell$ using the observations available till that episode. The optimistic estimate is built in two steps.

First, we use the available renormalized rewards and features $\{\tilde{r}_{h,s}, \phi_{A_{h,s},h,s}\}_{s=1,h=1}^{\ell-1,H}$ to create a ridge estimate of the underlying linear parameter $\theta$. Specifically,
\vspace{-.5em}
\begin{align}
\hat{\theta}_{\ell-1} &= \argmin_{\theta}\sum_{s=1}^{\ell-1} \gamma_{\rm decay}^{\ell-1-s} \sum_{h=1}^H  \left( \theta^\top \phi_{A_{h,s},h,s} - \tilde{r}_{{h,s}} \right)^2 + \frac{\lambda}{2}\|\theta\|_2^2 = (V_{\ell-1}^{\gamma_{\text{decay}}})^{-1} \sum_{s,h=1}^{\ell-1,H}  {\phi_{A_{h,s},h,s} \tilde{r}_{{h,s}}}\,,
\end{align}\vspace*{-1.5em}

where $V_{\ell-1}^{\gamma_{\text{decay}}} \defn \sum_{s=1}^{\ell-1} \gamma_{\rm decay}^{\ell-1-s} \sum_{h=1}^H \phi_{A_{h,s},h,s} \phi_{A_{h,s},h,s}^\top + \lambda \gamma_{\rm decay}^{\ell-1} I$ is an episodically discounted design matrix (aka Gram matrix) to tackle episodic context drifts.
We use $\hat{\theta}_{\ell-1}$ to construct an empirical estimate $\hat{\mu}_{0,\ell} \defn \langle \thetahat_{\ell-1}, \sum_a \pi_0(a) \sumh {\phi}_{a,h,\ell} \rangle$ of the reward of the baseline policy $\pi_0$ at the beginning of episode $\ell$. The proposed estimate extends the homoskedastic weighted linear regression~\citep{russac2019weighted,kim2026jointly} for non-stationary linear bandits to this episodic setting with heteroskedastic noise. 

Second, we compute the confidence width around the empirical estimate $\hat{\mu}_{0,\ell}$ using $V_{\ell-1}^{\gamma_{\text{decay}}}$. Then, we inflate the empirical mean estimate $\hat{\mu}_{0}$ using this confidence width and saturation information of arms to get an \textit{optimistic} estimate\vspace*{-.5em}
\begin{align}\label{eq:optimistic_baseline}
    \hat{\mu}_{0,\ell}^+ = \hat{\mu}_{0,\ell} + \sqrt{\beta_{\ell}(\delta_\ell)} \left\|\sum\nolimits_a \pi_0(a) \bar{\phi}_{a,\ell}\right\|_{(V_{\ell-1}^{\gamma_{\text{decay}}})^{-1}}\,,
\end{align}\vspace*{-1em}

where the confidence width ${\beta_{\ell}(\delta_\ell)} \defn \left(\sqrt{\log\frac{\det V_{\ell-1}^{\gamma_{\text{decay}}}}{\det V_0} + 2\log\frac{1}{\delta_\ell}} + \sqrt{\lambda}S_*\right)^2$.
Being optimistic in terms of base policy mean is being \textit{pessimistic about constraint satisficing}. This design resonates with the existing safe bandit literature~\citep{pacchiano2021stochastic, das2024learning}. 

Finally, this yields the constraint threshold $\tau_\ell = (1-\epsilon)\hat{\mu}_{0,\ell}^+$ for the present episode.    

\textit{Phase 3: The baseline adaptive MED step.} In beginning of this phase, we play an arm according to Algorithm~\ref{alg:med_step}. The sampling strategy uses Minimum Empirical Divergence at its core. The novelty lies in the adaption for our episodic setting with non-stationary contexts and baseline. We first compute the empirical best arm and estimate the gaps $\estgap$. We use these estimated gaps to calculate exponential weights over all arms as\vspace*{-.5em}
\begin{align}
        f_h(a)&= \exp\!\left(
        -{\hat{\Delta}_{a,h,\ell}^{2}}/
        \left({\beta_{h,\ell}(\alpha_{h,\ell})\|\phi_{\hat{a}_{h,\ell},h,\ell}-\phi_{a,h,\ell}\|^{2}_{(V_{h,\ell}^{\gamma})^{-1}}}\right)\right)\\
        \tilde{f}_h(a)&= f_h(a)\cdot \exp\!\Bigl(-\nu_{h,\ell} \cdot\max\{0,\,\tau_\ell - \mathrm{LCB}(a,h,\ell)\}\Bigr)\label{eq:design}
\end{align}\vspace*{-2em}

To make the weights constraint-aware, we introduce a Lagrangian dual based penalty with Lagrangian multiplier $\nu_{h,\ell}$. Specifically, we show in Lemma~\ref{lem:nu-bound} this dual parameter grows at a rate $\bigO(\sqrt{\ell h})$. Next, we rescale the arms using normalized features. We introduce a novel baseline-aware {\color{red}ApproxDesign()} algorithm (refer Section~\ref{supp;app_design}) or $C_{\rm opt}$-optimal design that ensures saturation of the true best arm efficiently, i.e.,  $\qquad\left\|a_{h,\ell}^*\right\|_{(V\left(p_h\right)^{\discount})^{-1}}^2 \leq C_{\rm opt}\bigO\left(d \log (d)\right).$

Though the baseline-awareness is adapted in the optimal design via $q_h(a)$. While we compute the final design probability in Line 9, we show in Lemma~\ref{lem:prob_lb}, the denominator is lower bounded by a dual-penalised term. Specifically $\sum_b q_h(b)\tilde{f}_h(b) > \alpha_{\rm opt}e^{-\nu_{h,\ell}}.$ This prevents unsafe arms from saturating and efficiently discards them. Finally, upon checking the saturation status of each arm and construct the augmented arm set as $B_{h,\ell} = \bigl\{a:\|\phi_{a,h,\ell}\|^{2}_{(V_{h,\ell}^{\gamma_{\rm decay}})^{-1}}>1\bigr\}$. Any arm in this set has not been explored enough to be sure of it's mean parameter, thus we pull this arm with a toss of a fair coin. Otherwise we play by the design probabilities. At the end, we observe the feedback signal, normalize the reward we compute. Then we check whether the constraint has been violated by the playing arm and update the Lagrangian multiplier $\nu_{h,\ell}$ accordingly.

\vspace{-.5em}\section{Regret Analysis of \framework}\label{sec:regret}\vspace{-.5em}
In this section, we provide upper bounds on both instance dependent regret, and expected constraint violation for the strategy \framework, along with very high-level proof intuitions. 
\begin{assumption}[Boundedness]\label{ass:gap_upper_bound} 
 1. $\max\limits_{a,h,\ell}\|\phi_{a,h,\ell}\|_2 \leq 1$,  
    2.  $\max\limits_{h}\|{\bf p}_h\|_2 \leq 1$, and 
   3. $\max\limits_{a,h,\ell}\Delta_{a,h,\ell} \leq B$.
\end{assumption}

\textbf{Standard Regret.} We first decompose the original regret in Equation~\ref{eq:regret} into two parts based on the event $\cV_{h,\ell}(A_{h,\ell}) = \left\{{\bf m}_{h,\ell}(A_{h,\ell})<(1-\epsilon)\mathbb{E}_{a\sim \pi_0}[{\bf m}_{h,\ell}(a)] \right\}$ as:
\begin{align}\label{eq:regret_decomp}
    \Reg =\underbrace {\expect \left[\suml\sumh \truegap \indicator\left\{ \overline{\cV_{h,\ell}(A_{h,\ell})}\right\} \right]}_{\defn \Reg^{\rm safe}}+ \underbrace{\expect \left[ \suml\sumh \truegap \indicator\left\{ \cV_{h,\ell}(A_{h,\ell}) \right\}\right]}_{\defn \Reg^{\rm viol}} 
\end{align}

The novel decomposition in Equation~\eqref{eq:regret_decomp} allows to analyze the regret of \framework~ in two stages: 1. bound on sum of suboptimality over time under the event that we do not pull any truly unsafe arm, 2. under the event we pull arms that are unsafe.

\textit{$\Reg^{\rm safe}$.} We formalize the regret upper bound guaranty under safety constraint  below in Theorem~\ref{thm:opt_reg_id}. 

\begin{tcolorbox}[top=1pt,bottom=1pt,left=2pt,right=2pt]
\begin{theorem}[Instance Dependent Regret Upper Bound under Constraint Satisfaction]\label{thm:opt_reg_id}
Under Assumption~\ref{ass:cont_indep},\ref{ass:linear}, and \ref{ass:gap_upper_bound}, Algorithm~\ref{alg:episodic-pref-imed-full} and~\ref{alg:med_step} jointly satisfies --
    \begin{align*}
        \Reg^{\rm safe} = \tilde{\bigO}\left( \kappa\frac{d^2\left(\log^2(LH)+\log\left( \Sigma_{LH}\right)\right)}{\Delta_{\min}} \right)
    \end{align*}
\end{theorem}

\end{tcolorbox}
\textit{Proof Concept.} The main challenge while analysing regret in our setting lies in elegantly incorporate the constraint violation event, and also to capture the effect of the episodic drift in context. To prove upper bound on the regret under constraint satisficing, we adapt the proof structure of~\cite{balagopalan2024minimum} to our episodic interaction setting with known, but shifting reward noise variance. The adaptation comes in two folds: 1. while analysing $\Reg^{\rm safe}$ the first challenge is handled by controlling the true gap $\truegap$ using peeling technique around $(\Dpi + \epsilon \mu_{\bpi_0})$, instead of the maximum gap $B$, 2. we realise the effect drifting context comes implicitly through the non-stationary noise variance involved in the feedback signal $Y_{A_{h,\ell},h,\ell}$. Thus, the log-dependency on sum of variances in Theorem~\ref{thm:opt_reg_id} comes naturally while we try to bound the saturation event of feature of $A_{h,\ell}$. i.e., $\expect\left[ \|\phi_{A_{h,\ell}}\|_{(\gram{h}{\ell})^{-1}} > 1\right]$. We prove this by using~\citep[Lemma 4]{abbasi2011improved} . For more details, refer to Section~\ref{supp:regret_ub} in the supplementary materials.   

\textit{$\Reg^{\rm viol}$.} This part of the whole regret is about characterising the suboptimality cost of the arms that are unsafe in true sense. We formally state the upper bound for this part below:

\begin{tcolorbox}[top=1pt,bottom=1pt,left=2pt,right=2pt]
\begin{theorem}[Instance Dependent Regret Upper Bound under Constraint Violation]\label{thm:reg_const_id}
Under Assumption~\ref{ass:cont_indep},\ref{ass:linear}, and \ref{ass:gap_upper_bound}, Algorithm~\ref{alg:episodic-pref-imed-full} and~\ref{alg:med_step} jointly satisfies --
    \begin{align*}
        \Reg^{\rm viol} =\tilde{\bigO}\left( \sigma_{\max}^2\frac{B-\Dpi}{(\Dpi+\epsilon\mu_0)^2} d^2\left(\log^2(LH)+\log\left( \Sigma_{LH}\right)\right)\right)
    \end{align*}
    
\end{theorem}

\end{tcolorbox}

\textit{Proof Concept.} Existing literature~\citep{bian2024indexed,balagopalan2024minimum} use peeling technique to leverage a lower bound on true gap that in turn, helps upper bounding the pulling probability of arm $A_{h,\ell}$. From the definition of the event $\cV_{h,\ell}(\alpha_{h,\ell)}$, we immediately show $\truegap > \Dpi+\epsilon\mu_{\bpi_0}$. Thus, 1. we do not need additional peeling over the value of $\truegap$, 2. we again leverage this natural lower bound while handling the conditioning event on concentration of $\thetahat_{h,\ell}$ around true parameter $\theta$. For the detailed proof, refer to Section~\ref{supp:regret_ub}.

\textit{Final Regret.} We combine $\Reg^{\rm safe}$ and $\Reg^{\rm viol}$ to state the final regret upper bound of \framework:

\begin{tcolorbox}[top=1pt,bottom=1pt,left=2pt,right=2pt]
\begin{theorem}[Instance Dependent Regret Upper Bound of~\framework]\label{thm:mainmed}
Let us denote $ \frac{1}{\tilde{\Delta}} \defn \max\left\{ \frac{1}{\Delta_{\min}},\frac{B-\Dpi}{(\Dpi+\epsilon\mu_0)^2}\right\}$. Under Assumption~\ref{ass:cont_indep},\ref{ass:linear}, and \ref{ass:gap_upper_bound}, \framework~ exhibits
    \begin{align*}
        \Reg = \tilde{\bigO}\left( \kappa \frac{1}{\tilde{\Delta}}d^2\left(\log^2(LH)+\log\left(\Sigma_{LH}\right)\right)\right)\,.
    \end{align*}
\end{theorem}
\end{tcolorbox}

\textit{Discussion.} 1. While \citep{balagopalan2024minimum} achieves a regret of $\tilde{\bigO}\left( \kappa \frac{1}{{\Delta}_{\min}}d^2\log^2(LH)\right)$, our bound also depends on the constraint gap $\Delta_0$ due to the baseline policy and the condition number of the variances over the episodes due to context drifts. 2. While the safe contextual bandit literature provides a minimax bound on the regret under the assumption of a safe action~\cite{amani2019linear,pacchiano2021stochastic}, we remove the assumption due to access to a baseline policy while achieving a problem-dependent regret upper bound as well as a minimax regret bound of same order. 3. On the other hand, if we consider the implication of our results for heteroskedastic linear bandits, we observe that we derive a problem-dependent regret bound under heteroskedasticity while the present literature focuses on minimax regret~\citep{kirschner2018information,kim2026jointly}.


\textit{Constraint Violation.} It is imperative for a safe bandit algorithm~\citep{amani2019linear,pacchiano2021stochastic} to have an upper bound on expected number of true constraint violations (defined in Equation~\eqref{eq:constraints} for the whole interaction. For \framework, we state this guaranty formally below:

\begin{tcolorbox}[top=1pt,bottom=1pt,left=2pt,right=2pt]
\begin{theorem}[Expected Constraint Violation of~\framework]\label{thm:cont_viol}

Under Assumption~\ref{ass:cont_indep},\ref{ass:linear}, and \ref{ass:gap_upper_bound}, \framework~ suffers expected number of constraint violations $\mathrm{Violation}(\bpi_0) = \tilde{\bigO}\left( d\right)$.
\end{theorem}
\end{tcolorbox}

\textit{Proof Concept.} In Lemma~\ref{lem:gap-lb}, we prove under constraint violation, the true gap is lower bounded by the gap subject to base-policy $\bpi_0$. That means $\truegap \geq \Dpi$. While analysing standard regret, we decompose $\Reg^{\rm viol}$ into two parts,
\begin{align*}
    \Reg^{\rm viol} = \Dpi\underbrace{\expect\left[ \suml\sumh \indicator\left\{ \cV_{h,\ell}(A_{h,\ell})\right\} \right]}_{\text{Violation}(\bpi_{0})} + \expect\left[ \suml\sumh (\truegap-\Dpi)\indicator\left\{ \cV_{h,\ell}(A_{h,\ell})\right\} \right]
\end{align*}
one of which is the expected constraint violation. We prove upper bound on this by simply using Corollary~\ref{cor:epc_rescaled} adapted~\citep[Lemma C.2]{jun2024noise}. We discuss it at length in Section~\ref{supp:regret_ub}.

\textit{Implication.}
(1) The variance dependency is captured by $\kappa$ and $\Sigma_{LH} \defn \suml{\discount}^{L-\ell} \sumh \bar{\sigma}^2_{h,\ell}$ appearing in the logarithmic term. This logarithmic part is reminiscent of the term $\sum_{t}\omega_t^2$ appearing in 
\cite[Thm.1]{russac2019weighted}, further adapted here to the heteroskedastic setup.

(2) Our bound is  instance-dependent, scaling as polylog of $T$, contrasting with minimax regret bounds.  From \cite{he2025variance,kim2026jointly}, the minimax lower bound is of order $d\sqrt{\sum_{t}\sigma_t^2}$ where $\sigma_t^2$ is the variance of reward $r_t$. 
Compare to the simplified setting of \cite{balagopalan2024minimum}, we replace the term $1/\Delta_{\min}$ with  the larger $\kappa/\tilde \Delta$, mainly due to the safety constraint, and the handling of heteroskedastic variance.
We conjecture the dependency on $\kappa$ might be further reduced.

(3) The intricate threshold within $\tilde \Delta$ appears as the constraint induced by $\pi_0$  does not vanish even as $\epsilon\to1$. Indeed, even in this case, the 
mean rewards must stay positive. 
We pay the largest regret when $\Delta_{\pi_0}=0$,  forcing a maximal factor $B/(\epsilon^2\mu_0^2)$, while as $\Delta_{\pi_0}$ approaches the largest gap $B$,
we recover the leading factor $1/\Delta_{\min}$ from the unconstrained setup~\citep{balagopalan2024minimum}.



\section{Numerical Experiments}\label{sec:xps}

We evaluate \framework~and \texttt{Dri-IMED}, a deterministic variant of our algorithm that we do not theoretically analyse (Appendix \ref{supp:dri-imed}), on a synthetic episodic contextual linear bandit with preference feedback and context drift, focusing on gradual and periodic drift regimes. Full environment details, additional drift regimes, ablation studies, and hyperparameter settings are provided in Appendix~\ref{app:experiments}.

\textbf{Baselines.}
We compare against four stationary linear bandit algorithms that ignore both drift and preference structure: \texttt{OFUL}~\citep{abbasi2011improved}, \texttt{LinMED}~\citep{balagopalan2024minimum}, \texttt{LinIMED}~\citep{bian2024indexed}, and \texttt{LinTS}~\citep{agrawal2013thompson}. To our knowledge, no prior algorithm addresses this combined setting. Thus, the baselines serve as the natural reference points\footnote{Source code can be found on \url{https://github.com/riiswa/context_drift_lin_bandits/}}.

\textbf{Environment.}
We consider a synthetic environment with $H=10$ users, $K=5$ arms, and $L=1000$ episodes; full parameter details are given in Appendix~\ref{app:experiments} (Algorithm~\ref{alg:env-gen}). We study two drift regimes, both parametrised by magnitude $\kappa=100$: \emph{gradual drift}, where the context scale grows linearly as $f(\ell) = \kappa \cdot \ell/L$, producing a smooth and monotone increase in noise; and \emph{periodic drift}, where $f(\ell) = \kappa \cdot \frac{1}{2}(1 + \sin(2\pi\ell / (L/50)))$, producing recurring fluctuations that repeatedly challenge the constraint mechanism. All results are averaged over 128 independent seeds.

\begin{figure}[t!]
    \centering
    \includegraphics[width=\textwidth]{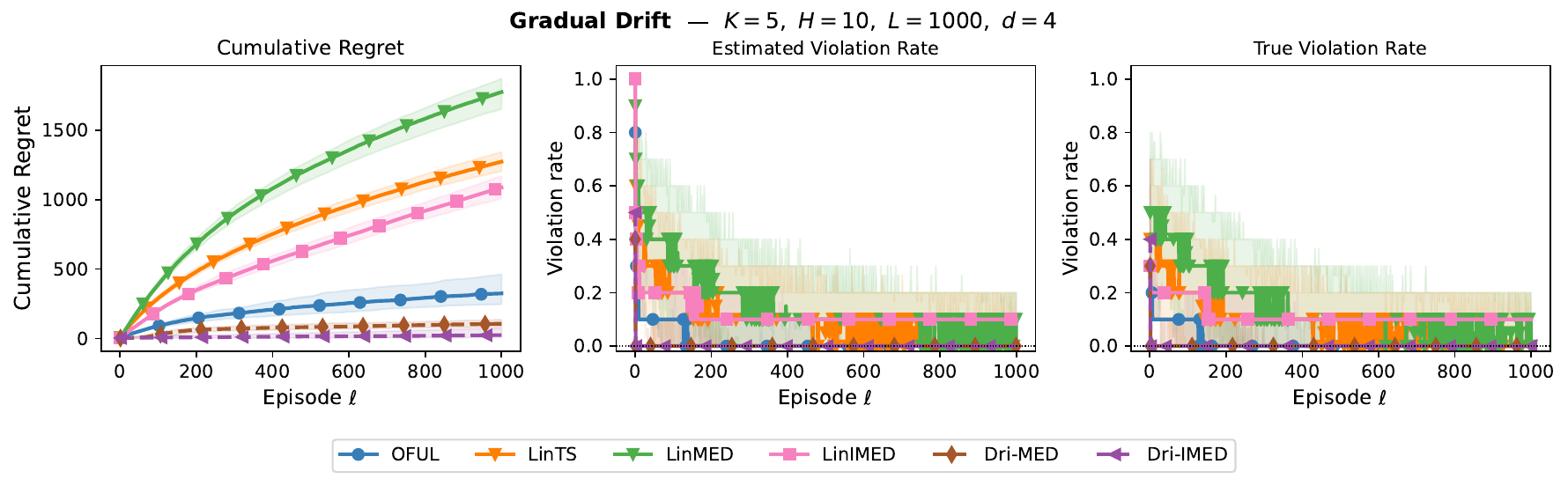}\\[4pt]
    \includegraphics[width=\textwidth]{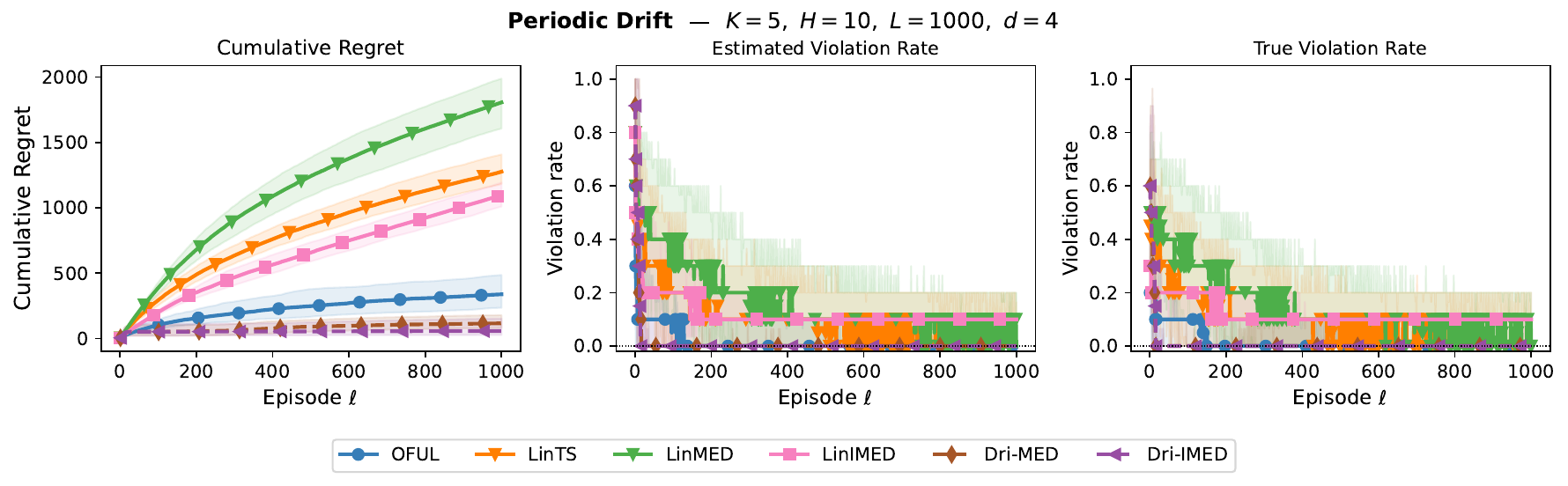}
    \caption{Cumulative regret (left), estimated violation rate (center), and true violation rate (right) for gradual and periodic drift. Shaded bands show the 5--95\% quantile range over 128 seeds. Results for no-drift and abrupt drift are provided in Appendix~\ref{app:experiments}.}   \label{fig:results_main}
\end{figure}

\textbf{Regret Evolution and Constraint Violation.}
Figure~\ref{fig:results_main} reports cumulative regret for both drift regimes. \framework~and \texttt{Dri-IMED} achieve substantially lower regret than all baselines, with the gap widening as episodes progress. Among the baselines, \texttt{OFUL} performs best yet still incurs regret an order of magnitude larger than that of our methods. \texttt{LinMED} and \texttt{LinTS} perform the worst, confirming that ignoring heteroskedasticity and preference structure is costly. Performance is stable across both drift types, validating that the drift-adaptive discounted regression successfully absorbs the non-stationarity, whether it is smooth or oscillatory.
Both \framework~and \texttt{Dri-IMED} rapidly drive the true constraint violation rate to zero after an initial exploration phase, while all stationary baselines exhibit persistent violations throughout. 

\textbf{Design Probabilities.}
The arm allocation analysis (Appendix~\ref{app:experiments}, Figure~\ref{fig:arm_allocation}) reveals a clear behavioural distinction between algorithms. \framework~concentrates almost all pull mass on the oracle-optimal arm $a^\star_h$ for every user. In contrast, \texttt{LinMED} spreads significant mass across suboptimal arms, unable to distinguish arm quality under the heteroskedastic preference structure it ignores.

For more experimental results and ablation studies, refer to Section~\ref{app:experiments} in supplementary materials.

\vspace*{-1em}\section{Discussions, Limitations, and Future Works}
We study a class of sequential experimental problems with a group of users having heterogeneous preferences and temporally evolving context distributions as a linear contextual bandit with heteroskedastic and non-stationary noise but stationary mean. In this context, we propose \framework{} that achieves (a) validity, i.e. a distribution of scores over its actions, (b) reliability, i.e. improved performance than a baseline action/policy applied over a control group, and (c) efficiency, i.e. low regret. We prove that \framework{} achieves logarithmic regret while exhibiting better numerical performance than the existing baselines.

The main limitation of this work is the stationary reward mean assumption. Though this assumption with heteroskedastic noise holds true for multiple applications, it would interesting to generalise our analysis for non-stationary means. In addition, the present MED strategies work only for sub-Gaussian rewards, whereas in real-life experiments, the noise distribution can exhibit heavy tails. Designing heavy-tail robust MED algorithms in this context, and in general remain open questions.
\subsubsection*{Broader Impact Statement}
\label{sec:broaderImpact}
This work mainly aimed for theoretical advancements in the safe, heteroskedastic contextual bandit literature. The authors does not see any potential negative impact that should be mentioned here.

\subsubsection*{Acknowledgments}
\label{sec:ack}
The authors would like to acknowledge PEPR project FOUNDRY (ANR23-PEIA-0003) for their support. DB and UD would like to acknowledge ANR JCJC project REPUBLIC (ANR-22-CE23-0003-01). We would also like to acknowledge the anonymous reviewers for their constructive feedback. Authors are members of the Inria team-project Scool.

\bibliography{main}

@article{balagopalan2024minimum,
  title={Minimum empirical divergence for sub-gaussian linear bandits},
  author={Balagopalan, Kapilan and Jun, Kwang-Sung},
  journal={arXiv preprint arXiv:2411.00229},
  year={2024}
}

@incollection{diggle2011,
    author = {Diggle, Peter J. and Chetwynd, Amanda G.},
    isbn = {9780199543182},
    title = {5 Experimental design: agricultural field experiments and clinical trials},
    booktitle = {Statistics and Scientific Method: An Introduction for Students and Researchers},
    publisher = {Oxford University Press},
    year = {2011},
    month = {08},
    doi = {10.1093/acprof:oso/9780199543182.003.0005},
    eprint = {https://academic.oup.com/book/0/chapter/195525344/chapter-pdf/43895872/acprof-9780199543182-chapter-5.pdf},
}

@article{neyman,
  title={On the Application of Probability Theory to Agricultural Experiments. Essay on},
  author={Neyman, Jerzy},
  journal={Statistical Science},
  volume={5},
  number={4,465-480},
  year={1923},
  note={reprinted in 1990}
}

@article{bian2024indexed,
  title={Indexed minimum empirical divergence-based algorithms for linear bandits},
  author={Bian, Jie and Tan, ntnt YF},
  journal={arXiv preprint arXiv:2405.15200},
  year={2024}
}

@article{honda2015non,
  title={Non-asymptotic analysis of a new bandit algorithm for semi-bounded rewards},
  author={Honda, Junya and Takemura, Akimichi},
  journal={The Journal of Machine Learning Research},
  volume={16},
  number={1},
  pages={3721--3756},
  year={2015},
  publisher={JMLR. org}
}

@article{baudry2023fast,
  title={Fast asymptotically optimal algorithms for non-parametric stochastic bandits},
  author={Baudry, Dorian and Pesquerel, Fabien and Degenne, R{\'e}my and Maillard, Odalric-Ambrym},
  journal={Advances in Neural Information Processing Systems},
  volume={36},
  pages={11469--11514},
  year={2023}
}

@article{qin2023kullback,
  title={Kullback-leibler maillard sampling for multi-armed bandits with bounded rewards},
  author={Qin, Hao and Jun, Kwang-Sung and Zhang, Chicheng},
  journal={Advances in Neural Information Processing Systems},
  volume={36},
  pages={60514--60526},
  year={2023}
}

@article{qin2025achieving,
  title={Achieving adaptivity and optimality for multi-armed bandits using Exponential-Kullback Leibler Maillard Sampling},
  author={Qin, Hao and Jun, Kwang-Sung and Zhang, Chicheng},
  journal={arXiv preprint arXiv:2502.14379},
  year={2025}
}

@article{honda2011asymptotically,
  title={An asymptotically optimal policy for finite support models in the multiarmed bandit problem},
  author={Honda, Junya and Takemura, Akimichi},
  journal={Machine Learning},
  volume={85},
  number={3},
  pages={361--391},
  year={2011},
  publisher={Springer}
}

@article{russac2019weighted,
	title={Weighted linear bandits for non-stationary environments},
	author={Russac, Yoan and Vernade, Claire and Capp{\'e}, Olivier},
	journal={Advances in Neural Information Processing Systems},
	volume={32},
	year={2019}
}

@article{faury2021regret,
  title={Regret bounds for generalized linear bandits under parameter drift},
  author={Faury, Louis and Russac, Yoan and Abeille, Marc and Calauz{\`e}nes, Cl{\'e}ment},
  journal={arXiv preprint arXiv:2103.05750},
  year={2021}
}

@article{he2025variance,
  title={Variance-Dependent Regret Lower Bounds for Contextual Bandits},
  author={He, Jiafan and Gu, Quanquan},
  journal={arXiv preprint arXiv:2503.12020},
  year={2025}
}

@inproceedings{kirschner2018information,
  title={Information directed sampling and bandits with heteroscedastic noise},
  author={Kirschner, Johannes and Krause, Andreas},
  booktitle={Conference On Learning Theory},
  pages={358--384},
  year={2018},
  organization={PMLR}
}

@inproceedings{abeille2017linear,
  title={Linear thompson sampling revisited},
  author={Abeille, Marc and Lazaric, Alessandro},
  booktitle={Artificial Intelligence and Statistics},
  pages={176--184},
  year={2017},
  organization={PMLR}
}

@article{durand2018streaming,
  title={Streaming kernel regression with provably adaptive mean, variance, and regularization},
  author={Durand, Audrey and Maillard, Odalric-Ambrym and Pineau, Joelle},
  journal={Journal of Machine Learning Research},
  volume={19},
  number={17},
  pages={1--34},
  year={2018}
}

@inproceedings{bregere2019target,
  title={Target tracking for contextual bandits: Application to demand side management},
  author={Br{\'e}g{\`e}re, Margaux and Gaillard, Pierre and Goude, Yannig and Stoltz, Gilles},
  booktitle={International Conference on Machine Learning},
  pages={754--763},
  year={2019},
  organization={PMLR}
}

@article{kim2026jointly,
  title={A Jointly Efficient and Optimal Algorithm for Heteroskedastic Generalized Linear Bandits with Adversarial Corruptions},
  author={Kim, Sanghwa and Lee, Junghyun and Yun, Se-Young},
  journal={arXiv preprint arXiv:2602.10971},
  year={2026}
}

@inproceedings{feng2025satisficing,
  title={Satisficing Regret Minimization in Bandits},
  author={Feng, Qing and Ma, Tianyi and Zhu, Ruihao},
  booktitle={The Thirteenth International Conference on Learning Representations},
  year={2025}
}

@article{amani2019linear,
  title={Linear stochastic bandits under safety constraints},
  author={Amani, Sanae and Alizadeh, Mahnoosh and Thrampoulidis, Christos},
  journal={Advances in Neural Information Processing Systems},
  volume={32},
  year={2019}
}

@inproceedings{pacchiano2021stochastic,
  title={Stochastic bandits with linear constraints},
  author={Pacchiano, Aldo and Ghavamzadeh, Mohammad and Bartlett, Peter and Jiang, Heinrich},
  booktitle={International conference on artificial intelligence and statistics},
  pages={2827--2835},
  year={2021},
  organization={PMLR}
}

@inproceedings{das2024learning,
  title={Learning to explore with Lagrangians for bandits under unknown constraints},
  author={Das, Udvas and Basu, Debabrota},
  booktitle={Seventeenth European Workshop on Reinforcement Learning},
  year={2024}
}

@article{krause2011contextual,
  title={Contextual gaussian process bandit optimization},
  author={Krause, Andreas and Ong, Cheng},
  journal={Advances in neural information processing systems},
  volume={24},
  year={2011}
}

@inproceedings{lattimore2017end,
  title={The end of optimism? an asymptotic analysis of finite-armed linear bandits},
  author={Lattimore, Tor and Szepesvari, Csaba},
  booktitle={Artificial Intelligence and Statistics},
  pages={728--737},
  year={2017},
  organization={PMLR}
}

@inproceedings{chowdhury2017kernelized,
  title={On kernelized multi-armed bandits},
  author={Chowdhury, Sayak Ray and Gopalan, Aditya},
  booktitle={International Conference on Machine Learning},
  pages={844--853},
  year={2017},
  organization={PMLR}
}

@inproceedings{valko2014spectral,
  title={Spectral bandits for smooth graph functions},
  author={Valko, Michal and Munos, R{\'e}mi and Kveton, Branislav and Koc{\'a}k, Tom{\'a}{\v{s}}},
  booktitle={International conference on machine learning},
  pages={46--54},
  year={2014},
  organization={PMLR}
}

@phdthesis{abeille2017exploration,
  title={Exploration-exploitation with Thompson sampling in linear systems},
  author={Abeille, Marc},
  year={2017},
  school={Universit{\'e} de Lille 1}
}

@inproceedings{jin2021almost,
  title={Almost optimal anytime algorithm for batched multi-armed bandits},
  author={Jin, Tianyuan and Tang, Jing and Xu, Pan and Huang, Keke and Xiao, Xiaokui and Gu, Quanquan},
  booktitle={International Conference on Machine Learning},
  pages={5065--5073},
  year={2021},
  organization={PMLR}
}

@article{perchet2016batched,
  title={Batched bandit problems},
  author={Perchet, Vianney and Rigollet, Philippe and Chassang, Sylvain and Snowberg, Erik},
  journal={The Annals of Statistics},
  pages={660--681},
  year={2016},
  publisher={JSTOR}
}

@article{gao2019batched,
  title={Batched multi-armed bandits problem},
  author={Gao, Zijun and Han, Yanjun and Ren, Zhimei and Zhou, Zhengqing},
  journal={Advances in Neural Information Processing Systems},
  volume={32},
  year={2019}
}

@article{zhang2020inference,
  title={Inference for batched bandits},
  author={Zhang, Kelly and Janson, Lucas and Murphy, Susan},
  journal={Advances in neural information processing systems},
  volume={33},
  pages={9818--9829},
  year={2020}
}

@inproceedings{agrawal2013thompson,
  title={Thompson sampling for contextual bandits with linear payoffs},
  author={Agrawal, Shipra and Goyal, Navin},
  booktitle={International conference on machine learning},
  pages={127--135},
  year={2013},
  organization={PMLR}
}

@book{tanner2023experimenting,
  title={Experimenting on the Farm: Introduction to Experimental Design},
  author={Tanner, Kari Christine and Jones, Gordon B and Verhoeven, Elizabeth Claire},
  year={2023},
  publisher={Oregon State University Extension Service}
}

@book{gomez1984statistical,
  title={Statistical procedures for agricultural research},
  author={Gomez, Kwanchai A and Gomez, Arturo A},
  year={1984},
  publisher={John wiley \& sons}
}

@article{hazan2016introduction,
  title={Introduction to online convex optimization},
  author={Hazan, Elad},
  journal={Foundations and Trends in Optimization},
  volume={2},
  number={3-4},
  pages={157--325},
  year={2016},
  publisher={Emerald Publishing Limited}
}

@article{abbasi2011improved,
  title={Improved algorithms for linear stochastic bandits},
  author={Abbasi-Yadkori, Yasin and P{\'a}l, D{\'a}vid and Szepesv{\'a}ri, Csaba},
  journal={Advances in neural information processing systems},
  volume={24},
  year={2011}
}

@article{jun2024noise,
  title={Noise-adaptive confidence sets for linear bandits and application to bayesian optimization},
  author={Jun, Kwang-Sung and Kim, Jungtaek},
  journal={arXiv preprint arXiv:2402.07341},
  year={2024}
}

@inproceedings{zhu2022contextual,
  title={Contextual bandits with large action spaces: Made practical},
  author={Zhu, Yinglun and Foster, Dylan J and Langford, John and Mineiro, Paul},
  booktitle={International Conference on Machine Learning},
  pages={27428--27453},
  year={2022},
  organization={PMLR}
}

@article{tirinzoni2020asymptotically,
  title={An asymptotically optimal primal-dual incremental algorithm for contextual linear bandits},
  author={Tirinzoni, Andrea and Pirotta, Matteo and Restelli, Marcello and Lazaric, Alessandro},
  journal={Advances in Neural Information Processing Systems},
  volume={33},
  pages={1417--1427},
  year={2020}
}

@article{aylmer1926arrangement,
  title={The arrangement of field experiments},
  author={Fisher, Ronald},
  journal={Journal of the Ministry of Agriculture},
  volume={33},
  pages={503--515},
  year={1926},
  publisher={Ministry of Agriculture and Fisheries}
}

@book{rangaswamy1995text,
  title={A text book of agricultural statistics},
  author={Rangaswamy, R},
  year={1995},
  publisher={New age international},
}

@article{banerjee2020theory,
  title={A theory of experimenters: Robustness, randomization, and balance},
  author={Banerjee, Abhijit V and Chassang, Sylvain and Montero, Sergio and Snowberg, Erik},
  journal={American Economic Review},
  volume={110},
  number={4},
  pages={1206--1230},
  year={2020},
  publisher={American Economic Association 2014 Broadway, Suite 305, Nashville, TN 37203}
}

@article{kohavi2015online,
  title={Online controlled experiments and A/B tests},
  author={Kohavi, Ron and Longbotham, Roger},
  journal={Encyclopedia of machine learning and data mining},
  pages={1--11},
  year={2015},
  publisher={New York: Springer Publishing}
}

@article{raccuglia2016machine,
  title={Machine-learning-assisted materials discovery using failed experiments},
  author={Raccuglia, Paul and Elbert, Katherine C and Adler, Philip DF and Falk, Casey and Wenny, Malia B and Mollo, Aurelio and Zeller, Matthias and Friedler, Sorelle A and Schrier, Joshua and Norquist, Alexander J},
  journal={Nature},
  volume={533},
  number={7601},
  pages={73--76},
  year={2016},
  publisher={Nature Publishing Group UK London}
}

@book{hoshmand2018design,
  title={Design of experiments for agriculture and the natural sciences},
  author={Hoshmand, Reza},
  year={2018},
  publisher={Chapman and Hall/CRC}
}

@book{sutton,
  title={Reinforcement learning: An introduction},
  author={Sutton, Richard S and Barto, Andrew G and others},
  year={1998},
  publisher={MIT press Cambridge}
}

@book{lattimore_szepesvari_2020, 
place={Cambridge}, title={Bandit Algorithms}, DOI={10.1017/9781108571401}, publisher={Cambridge University Press}, author={Lattimore, Tor and Szepesvári, Csaba}, year={2020}}

@article{reda2020machine,
  title={Machine learning applications in drug development},
  author={R{\'e}da, Cl{\'e}mence and Kaufmann, Emilie and Delahaye-Duriez, Andr{\'e}e},
  journal={Computational and structural biotechnology journal},
  volume={18},
  pages={241--252},
  year={2020},
  publisher={Elsevier}
}

@article{munro2021safety,
  title={Safety and immunogenicity of seven COVID-19 vaccines as a third dose (booster) following two doses of ChAdOx1 nCov-19 or BNT162b2 in the UK (COV-BOOST): a blinded, multicentre, randomised, controlled, phase 2 trial},
  author={Munro, Alasdair PS and Janani, Leila and Cornelius, Victoria and Aley, Parvinder K and Babbage, Gavin and Baxter, David and Bula, Marcin and Cathie, Katrina and Chatterjee, Krishna and Dodd, Kate and others},
  journal={The Lancet},
  volume={398},
  number={10318},
  pages={2258--2276},
  year={2021},
  publisher={Elsevier}
}

@article{mavrotas2009effective,
  title={Effective implementation of the $\varepsilon$-constraint method in multi-objective mathematical programming problems},
  author={Mavrotas, George},
  journal={Applied mathematics and computation},
  volume={213},
  number={2},
  pages={455--465},
  year={2009},
  publisher={Elsevier}
}

@article{feijer2025calibrated,
  title={Calibrated Recommendations with Contextual Bandits},
  author={Feijer, Diego and Abdollahpouri, Himan and Gupta, Sanket and Clare, Alexander and Wen, Yuxiao and Wasson, Todd and Dimakopoulou, Maria and Nazari, Zahra and Kretschman, Kyle and Lalmas, Mounia},
  journal={arXiv preprint arXiv:2509.05460},
  year={2025}
}

@article{saha2026one,
  title={One Good Source is All You Need: Near-Optimal Regret for Bandits under Heterogeneous Noise},
  author={Saha, Aadirupa and Bhat, Amith and Luo, Haipeng},
  journal={arXiv preprint arXiv:2602.14474},
  year={2026}
}

@article{legedza2001heterogeneity,
  title={Heterogeneity in phase I clinical trials: prior elicitation and computation using the continual reassessment method},
  author={Legedza, Anna TR and Ibrahim, Joseph G},
  journal={Statistics in Medicine},
  volume={20},
  number={6},
  pages={867--882},
  year={2001},
  publisher={Wiley Online Library}
}
\bibliographystyle{rlj}

\beginSupplementaryMaterials
\appendix
\section{Notations}
\begin{minipage}{\textwidth}
\centering
\begin{longtable}{ll}
\hline
\textbf{Notation} & \textbf{Description}  \\
\hline  
$\ell$ & Episode index $\in [1,L]$  \\
$h$ & User index $\in [1,H]$\\
$\theta$ & True reward parameter with $\|\theta\|_2\leq S_*$ \\
$\thetahat_{h,\ell}$ & Estimate of the true parameter $\theta$ for $h$-th user in $\ell$-th episode\\
$K$ & Number of actions\\
$\context$ & Context distribution at episode $\ell$\\
$\Phi(a)$ & $M\times d$ feature matrix for arm $a\in \cA$\\ 
$Y_{h,\ell}$ & $M$-dimensional feedback signal $\expect[Y_{h,\ell}\mid A_{h,\ell}=a,C_{h,\ell}= c] = \Phi(a)\theta$\\
$\bp_{h}$ & $d$-dimensional preference vector for $h$-th user\\
$r_{h,\ell}$ & Scalar reward for $h$-th user in $\ell$-th episode $= Y_{h,\ell}^{\top}\bp_h$
\\
$\phi_{a,h,\ell}$ & $d$-dimensional context for arm $a$ $= \frac{1}{\sigma_{a,h,\ell}}\Phi(a)^{\top}\bp_h$\\
$\Sigma_{a,h,\ell}$ & $= {\rm Cov}[Y_{h,\ell}\mid A_{h,\ell}=a,C_{h,\ell} = c]$\\
$\sigma_{a,h,\ell}$ & $={\rm Var}(\eta_{h,\ell}\mid A_{h,\ell} = a) $ (Variance of noise in the feedback, if arm $a$ is pulled)\\
$\sigma_{h,\ell}^2$ & $= \min_{a\in \cA} \sigma_{a,h,\ell}^2$\\
$\sigma_{\max}^2$ & $=\max_{a,h,\ell}\sigma^2_{a,h,\ell}$\\
$\sigma_{\min}^2$ & $=\min_{a,h,\ell}\sigma^2_{a,h,\ell}$\\
$\Delta_{a,h,l}$ & $=\theta^\top(\phi_{a^*_{\ell},h,\ell} - \phi_{a,h,\ell}) \ge 0$\\
$\hat{\Delta}_{a,h,l}$ & $= \thetahat_t^\top(\phi_{a^*_{\ell},h,\ell} - \phi_{a,h,\ell})$\\
$\pi_0$ & Baseline policy\\
$\Dpi$ & $= \min_{\ell\in[L]}\theta^{\top}(\phi_{a^*_{\ell},h,\ell}-\phi_{\pi_0})$\\
$\mu_{0,\ell}$ & Estimated reward of the baseline policy $\frac{1}{H}\sum_{h=1}^H \sum_a \pi_0(a) \langle \thetahat_{h,\ell}, \phi_{a,h,\ell}\rangle$\\
$\mu_{0,\ell}^+$ & Optimistic reward of baseline $\mu_{0,\ell} + \sqrt{\beta_{\ell}(\delta_\ell)} \left\|\frac{1}{H}\sum_h \sum_a \pi_0(a) \phi_{a,h,\ell}\right\|_{(V_{\ell}^{\gamma})^{-1}}$\\
$\tau_{\ell}$ & Constraint threshold $= (1-\epsilon) \mu_{0,\ell}^+$, for $\epsilon \ge 0$\\
$\tau$ & True constraint threshold $=(1-\epsilon) \mu_0$, where $\mu_0$ is the true mean reward for $\pi_0$\\
$\xi_{h,\ell}$ & Constraint violation indicator $= \indicator\left\{\langle \hat{\theta}_{h,\ell}, \phi_{A_{h,\ell},h} \rangle < \tau_\ell\right\}$\\
$\cV_{h,\ell}(A_{h,\ell)}$ & True constraint violating event $\{(h,\ell)\in[h,\ell]: \langle \theta,\phi_{A_{h,\ell},h}\rangle < \tau \}$\\
$\mu_0$ & True mean of base-policy $=\theta^{\top}\phi_{\bpi_0}$\\
\hline
\end{longtable}
\end{minipage}

For the supplementary materials, whenever we write $\underline{\Delta}_{a,h,\ell}$, it denotes the feature normalized true suboptimal gap. \clearpage
\section{Proof of Regret Upper Bound of Dri-MED (Algorithm~\ref{alg:episodic-pref-imed-full})} \label{supp:regret_ub}
\subsection{Good Event: Concentration and Confidence Width} 

We start by defining the following good events: 

\begin{align}
    \cG_1 \defn \left\{\forall h\ge1,\ell\ge 1: \|\theta - \thetahat_{h,\ell}\|_{\gram{h}{\ell}}^2 \leq \beta_{h,\ell}(\alpha_{h,\ell})  \right\}\label{eq:good_event_hl}
\end{align}


\begin{lemma}[Heteroscedastic weighted $(h,\ell)$-th confidence set]\label{lem:conf}
    Following the information acquisition rule with  action $A_{h,\ell} = a$ and it's rescaled feature vector $\phi_{a,h,\ell}$, i.e., $\gram{h}{\ell} = \gram{h-1}{\ell} + \phi_{a,h,\ell}\phi_{a,h,\ell}^{\top}$, we define
\begin{align*}
    {\beta_{h,\ell}(\alpha_{h,\ell})}^{1/2} \defn 
      \sqrt{2\log\frac{\det \gram{h}{\ell}}{\det V_0}
           + 2\log\frac{1}{\alpha_{h,\ell}}}
      + \sqrt{\lambda}\,S_* \,,
\end{align*}

where $\alpha_{h,\ell}\in (0,1)$ is to characterised later on. Then $\Prob(\cG_1) \ge 1 - \suml\sumh \alpha_{h,\ell}$.
\end{lemma}
\begin{proof}
     We apply the self-normalized martingale inequality of~\citep[Theorem 2]{abbasi2011improved} to the Gram matrix $\gram{h}{\ell}$. At round $(h,\ell)$, the noise involved in the reward signal $\eta_{h,\ell}$ is $\sigma^2_{a,h,\ell}$-sub-Gaussian and $\cF_{h,\ell}$-measurable. 
     The weighted process $\suml \gamma^\ell\sumh  \eta_{h,\ell}\phi_{A_{h,\ell},h,\ell}$ is a martingale, and the theorem yields
    the stated bound round-by-round. Thus, a union bound over all $(h,\ell)$ gives $\Prob(\cG_1) \ge
    1 - \suml\sumh \alpha_{h,\ell}$.
\end{proof}

\subsection{Regret Decomposition}

\begin{definition}[Constraint-violating rounds]
  Let $\cV_{h,\ell}(A_{h,\ell}) \defn \left\{ \langle \theta,\phi_{A_{h,\ell},h,\ell}\rangle < \tau \right\}$ denotes whether $(h,\ell)$ is a true constraint-violating index. 
\end{definition}
We should note, as we are normalising each feature by the noise variance, the gaps are also inherently scaled by the noise variance. Thus, we define $\truegap = \sigma_{A_{h,\ell},h,\ell} \truegaprs, \estgap = \sigma_{A_{h,\ell},h,\ell} \estgaprs $, and finally ${\Dpirs}{\sigma_{\min}}\leq\Dpi\leq{\Dpirs}{\sigma_{\max}} $.  
We decompose the total expected regret as--
\begin{align*}
    \Reg &=  \expect\!\left[\sum_{\ell=1}^{L}\sum_{h=1}^{H} \truegap\right]\\
    &=  \sigma_{\max}\expect\!\left[\sum_{\ell=1}^{L}\sum_{h=1}^{H} \truegaprs\indicator\left\{\overline{\cV_{h,\ell}(A_{h,\ell})}\right\}\right]+\sigma_{\max}\expect\!\left[\sum_{\ell=1}^{L}\sum_{h=1}^{H} \truegaprs\indicator\left\{{\cV_{h,\ell}(A_{h,\ell})}\right\}\right]\\
   &\defn \Reg^{\mathrm{safe}} + \Reg^{\mathrm{viol}} 
\end{align*}

\subsection{Part I: Regret Upper Bound under Constraint Violation}

\begin{reptheorem}{thm:reg_const_id}
Algorithm~\ref{alg:episodic-pref-imed-full} and~\ref{alg:med_step} jointly satisfies --
    \begin{align*}
        \Reg^{\rm viol} =\tilde{\bigO}\left(\sigma_{\max}^2 \frac{B-\Dpi}{(\Dpi+\epsilon\mu_0)^2} d^2\left(\log^2(LH)+\log\left( \suml {(\discount)}^{L-\ell} \sumh \sigma^2_{h,\ell}\right)\right)\right)
    \end{align*}
    
\end{reptheorem}
In this section, we provide a complete proof of the regret upper bound under constraint violation, i.e, upper bound on $\Reg^{\rm viol}$. The proof structure involves two steps:

\begin{enumerate}
    \item Then for the upper bound, we first define the conditioning event that will be used to decompose the regret expression further.
    \item We bound each part of the decomposition to get a final accumulated upper bound.
\end{enumerate}

\subsubsection{Defining the Conditioning Events}

\begin{definition}\label{def:cond_events}
    For any arm $a\in\cA$, we define the following conditioning events:
        \begin{equation}
        \begin{split}
            &\text{Concentration of gaps:}~\cU_{h,\ell}(a) \defn \left\{ \estgaprs \ge \frac{\truegaprs}{1+c} \right\} \quad\quad \text{for } c\geq 0 \\
            &\text{Saturation of arm features:}~ \cE_{h,\ell} \defn \left\{ |B_{h,\ell}|>0\right\}, \text{where } B_{h,\ell} = \bigl\{a:\|\phi_{a,h,\ell}\|^{2}_{(V_{h,\ell}^{\gamma_{\rm decay}})^{-1}}>1\bigr\}\\
            &\text{Concentration of parameter estimate:}~\cJ_{h,\ell}(a) \defn \left\{ \theta^{\top}\phi_{a_{h,\ell}^*,h,\ell} - \thetahat_{h,\ell}^{\top}\phi_{a_{h,\ell}^*,h,\ell} \leq \varepsilon_1 \right\}\quad \text{for } \varepsilon_1\geq 0\\
            &\text{Partial Saturation of arm features:}~ \cF_{h,\ell}(a) \defn \left\{ \|\phi_{a,h,\ell}\|_{(\gram{h}{\ell})^{-1}} > \varepsilon_2 \right\} \quad\quad \text{for } \varepsilon\geq 0                  
        \end{split}\label{eq:cond_events}
        \end{equation}
\end{definition}

The event $\cU_{h,\ell}(a)$ signifies how accurately we are being able to estimate the sub-optimality gaps for arm $a$, $\cE_{h,\ell}$ measures the saturation level for every arm at index $(h,\ell)$, $\cJ_{h,\ell}(a)$ assigns a precision in estimating the true parameter $\theta$, and finally $F_{h,\ell}$ sets a lower bound on the saturation (number of pulls) level of arm $a$.
 
\subsubsection{Decomposition of $\Reg^{\mathrm{viol}}$}
We rewrite $\truegaprs = \Dpirs + (\truegaprs - \Dpirs)$ for $(h,\ell) \in \cV$. 
Thus, we have:
\begin{align}\label{eq:constr-decomp}
  \Reg^{\mathrm{constr}}
  =& \sigma_{\max}\expect\!\left[\suml\sumh \truegaprs\indicator\left\{{\cV_{h,\ell}(A_{h,\ell})}\right\}\right]\notag\\
  =& \sigma_{\max}\Dpirs\underbrace{ \expect\!\left[\suml\sumh \indicator\left\{{\cV_{h,\ell}(A_{h,\ell})}\right\}\right]}_{\defn A_1} 
  + \sigma_{\max}\underbrace{\expect\!\left[\suml\sumh (\truegaprs-\Dpirs)\indicator\left\{{\cV_{h,\ell}(A_{h,\ell})}\right\}\right]}_{\defn A_2}
\end{align}
Both the terms in this decomposition are non-negative as Lemma~\ref{lem:gap-lb} yields a natural
lower bound on the regret due to constraint-violation $\Reg^{\rm constr} \ge \Dpi\expect\!\left[\suml\sumh \indicator\left\{{\cV_{h,\ell}(A_{h,\ell})}\right\}\right]$.

\textbf{Step 1: Upper bound on $A_1$.} We further decompose $A_1$ using the conditioning event $\cU_{h,\ell}(A_{h,\ell})$ as below:

\begin{align}
    A_1 =&  \expect\!\left[\suml\sumh \indicator\left\{{\cV_{h,\ell}(A_{h,\ell})}\right\}\right]\notag \\
    \leq& \underbrace{\expect\!\left[\suml\sumh \indicator\left\{{\cV_{h,\ell}(A_{h,\ell})}\right\}\indicator\{\cU_{h,\ell}(A_{h,\ell})\}\right]}_{\defn B_1} +  \underbrace{\expect\!\left[\suml\sumh \indicator\left\{{\cV_{h,\ell}(A_{h,\ell})}\right\}\indicator\{\overline{\cU_{h,\ell}(A_{h,\ell})}\}\right]}_{\defn B_2}
\end{align}

\textit{Upper bound on $B_1$.} 
We first divide the term $B_1$ in two parts: one under the good event $\cG_1$ and when $\cG_1$ is false. We denote them by $B_1\mid_{\cG_1}$ and $B_1\mid_{\overline{\cG_1}}$, respectively.

First, we bound $B_1\mid_{\cG_1}$ as follows:
\begin{align*}
    B_1\mid_{\cG_1} =& \expect\!\left[\suml\sumh \indicator\left\{{\cV_{h,\ell}(A_{h,\ell})}\right\}\indicator\{\cU_{h,\ell}(A_{h,\ell})\}\right]\\
    =& \expect\!\left[\suml\sumh \indicator\left\{\truegaprs > \Dpirs+\epsilon\mu_0\right\}\indicator\left\{\estgap \ge \frac{\truegaprs}{1+c}\right\}\right]\\
    =& \expect\!\left[\suml\sumh \indicator\left\{\estgap > \frac{\Dpirs+\epsilon\mu_0}{1+c}\right\}\right]\\
    \stackrel{(a)}{\leq}& \expect\!\left[\suml\sumh \indicator\left\{\estgap - \truegaprs> \frac{\Dpirs+\epsilon\mu_0}{1+c} - B\right\}\right]\\
    =& \expect\!\left[\suml\sumh \indicator\left\{  (\theta-\thetahat_{h,\ell})^{\top}(\phi_{A_{h,\ell},h,\ell}-\phi_{\bestarm,h,\ell})>\frac{\Dpirs+\epsilon\mu_0}{1+c} - B\right\}\right]\\
    \stackrel{(b)}{\leq}&\expect\!\left[\suml\sumh \indicator\left\{  {\beta_{h,\ell}(\alpha_{h,\ell})}^{1/2}\|\phi_{A_{h,\ell},h,\ell}-\phi_{\bestarm,h,\ell}\|_{{(\gram{h}{\ell})}^{-1}}>\frac{\Dpirs+\epsilon\mu_0}{1+c} - B\right\}\right]\\
    \leq &\expect\!\Bigg[\suml\sumh \indicator\Bigg\{  \|\phi_{A_{h,\ell},h,\ell}\|_{{(\gram{h}{\ell})}^{-1}}^2+\|\phi_{\bestarm,h,\ell}\|_{{(\gram{h}{\ell})}^{-1}}^2\\
    &\qquad\qquad\qquad>\qquad\frac{1}{{\beta_{h,\ell}(\alpha_{h,\ell})}}\left(\frac{\Dpirs+\epsilon\mu_0}{1+c} - B\right)^2\Bigg\}\Bigg]\\
    \stackrel{(c)}{\leq} &2\expect\!\left[\suml\sumh \indicator\left\{  \|\phi_{A_{h,\ell},h,\ell}\|_{{(\gram{h}{\ell})}^{-1}}^2>\frac{1}{{2\beta_{1,1}(\alpha_{1,1})}}\left(\frac{\Dpirs+\epsilon\mu_0}{1+c} - B\right)^2 \right\}\right]\\
    \leq& 2\expect\!\left[\suml\sumh \indicator\left\{  \|\phi_{A_{h,\ell},h,\ell}\|_{{(\gram{h}{\ell})}^{-1}}^2>\frac{1}{2\lambda S_*^2}\left(\frac{\Dpirs+\epsilon\mu_0}{1+c} - B\right)^2 \right\}\right]\,,
\end{align*}
where $(a)$ holds due to the Assumption~\ref{ass:gap_upper_bound}, and the fact that
\begin{align*}
    \left\{  \theta^{\top}\phi_{A_{h,\ell},h,\ell} < \tau \right\}=& \left\{  \theta^{\top}\phi_{A_{h,\ell},h,\ell} < (1-\epsilon)\theta^{\top}\phi_{\bpi_0} \right\}\\
    =&\left\{  \theta^{\top}\phi_{A_{h,\ell},h,\ell}-\theta^{\top}\feature{\bestarm}+\theta^{\top}\feature{\bestarm} -\theta^{\top}\phi_{\bpi_0}< -\epsilon \mu_0 \right\}\\
    =&\left\{ -\truegaprs + \theta^{\top}\feature{\bestarm}-\theta^{\top}\phi_{\bpi_0} < -\epsilon \mu_0 \right\}\\
    \subseteq &\left\{ -\truegaprs + \min_{\ell\in[L]}\theta^{\top}\feature{\bestarm}-\theta^{\top}\phi_{\bpi_0} < -\epsilon \mu_0 \right\}\\
    =&\left\{ -\truegaprs + \Dpirs < -\epsilon \mu_0 \right\}\\
    =&\left\{ \truegaprs > \Dpirs +\epsilon \mu_0 \right\}\,,
\end{align*}
$(b)$ holds due to the definition of good event $\cG_1$ in Equation~\eqref{eq:good_event_hl}, and finally $(c)$ holds because we apply Boole's inequality, $\forall h,\ell\ge 1, \beta_{h,\ell}(\alpha_{h,\ell})\geq \beta_{1,1}(\alpha_{1,1})$. 

It is easy to see if we put $\lambda = \frac{2S_*^2}{\min\left\{\left( \frac{\Dpirs+\epsilon\mu_0}{1+c} - B\right)^2 , \left( {\Dpirs+\epsilon\mu_0} - \frac{B}{1+c}\right)^2 \right\}}$, then 
\begin{align*}
    B_1\mid_{\cG_1} = \expect\!\left[\suml\sumh \indicator\left\{  \|\phi_{A_{h,\ell},h,\ell}\|_{{(\gram{h}{\ell})}^{-1}}^2>1\right\}\right]
\end{align*}

Leveraging the elliptical potential count lemma (Lemma~\ref{lemma:epc_count}) and Lemma~\ref{lem:conf}, we get the final bound on $B_1$ as:
\begin{align*}
    B_1 \leq B_1\mid_{\cG_1} + B_1\mid_{\overline{\cG_1}} \leq 6d\log\left(1+\frac{2}{\lambda}\right) + \suml\sumh \alpha_{h,\ell}
\end{align*}

\textit{Upper Bound on $B_2$.} We again divide the second term $B_2$ in two parts: one under the good event $\cG_1$ and otherwise. We similarly denote them by $B_2\mid_{\cG_1}$ and $B_2\mid_{\overline{\cG_1}}$, respectively.

First, we bound $B_2\mid_{\cG_1}$ as follows:
\begin{align*}
    B_2\mid_{\cG_1} =& \expect\!\left[\suml\sumh \indicator\left\{{\cV_{h,\ell}(A_{h,\ell})}\right\}\indicator\{\overline{\cU_{h,\ell}(A_{h,\ell})}\}\right]\\
    =&\expect\!\left[\suml\sumh \indicator\left\{\truegaprs > \Dpirs+\epsilon\mu_0\right\}\indicator\left\{\estgap < \frac{\truegaprs}{1+c}\right\}\right]\\
    \leq & \expect\!\left[\suml\sumh \indicator\left\{\estgap-\truegaprs \leq \frac{\truegaprs}{1+c} - \Dpirs-\epsilon\mu_0\right\}\right]\\
    \leq & \expect\!\left[\suml\sumh \indicator\left\{\estgap-\truegaprs \leq \frac{B}{1+c} - \Dpirs-\epsilon\mu_0\right\}\right]\\
    = &\expect\!\left[\suml\sumh \indicator\left\{\truegaprs-\estgap \geq  \Dpirs+\epsilon\mu_0-\frac{B}{1+c}  \right\}\right]\\
    \leq&2\expect\!\left[\suml\sumh \indicator\left\{  \|\phi_{A_{h,\ell},h,\ell}\|_{{(\gram{h}{\ell})}^{-1}}^2\geq\frac{1}{{2\beta_{h,\ell}(\alpha_{h,\ell})}}\left(\Dpirs+\epsilon\mu_0-\frac{B}{1+c}\right)^2 \right\}\right]\\
    \leq&2\expect\!\left[\suml\sumh \indicator\left\{  \|\phi_{A_{h,\ell},h,\ell}\|_{{(\gram{h}{\ell})}^{-1}}^2\geq\frac{1}{\lambda S_*^2}\left(\Dpirs+\epsilon\mu_0-\frac{B}{1+c}\right)^2 \right\}\right]
\end{align*}
Similar to the analysis $B_1$, we use Lemma~\ref{lemma:epc_count} with $\lambda = \frac{2S_*^2}{\min\left\{\left( \frac{\Dpirs+\epsilon\mu_0}{1+c} - B\right)^2 , \left( {\Dpirs+\epsilon\mu_0} - \frac{B}{1+c}\right)^2 \right\}}$ to get the final bound on $B_2$ as:
\begin{align*}
    B_2  \leq   6d\log\left( 1+\frac{2}{ \lambda}\right) + \suml\sumh \alpha_{h,\ell}\,.
\end{align*}

\textbf{Step 2: Upper Bound on $A_2$.} 


To find upper bound on $A_2$, we closely follow the proof structure of the standard regret in~\citep[Lemma 1 and Theorem 1]{balagopalan2024minimum} but remove the need of the peeling event. \citet{balagopalan2024minimum} leverage peeling to lower bound the gap $\truegap$, since it appears in the denominator of the pulling probability of arm $A_{h,\ell}$. We get the lower bound due to the presence of $\bpi_0$. 

We use the event $\cE_{h,\ell}$ to decompose $A_2$ as:
\begin{align*}
    A_2 =& \expect\!\left[\suml\sumh (\truegaprs-\Dpirs)\indicator\left\{{\cV_{h,\ell}(A_{h,\ell})}\right\}\right]\\
    \leq & \underbrace{\expect\!\left[\suml\sumh (\truegaprs-\Dpirs)\indicator\left\{{\cV_{h,\ell}(A_{h,\ell})}\right\}\indicator\left\{\cE_{h,\ell}\right\}\right]}_{\defn C_1} \\&\qquad\qquad+ \underbrace{\expect\!\left[\suml\sumh (\truegaprs-\Dpirs)\indicator\left\{{\cV_{h,\ell}(A_{h,\ell})}\right\}\indicator\left\{\overline{\cE_{h,\ell}}\right\}\right]}_{\defn C_2}
\end{align*}

\textit{Upper bound on $C_1$.} Now, we focus on the term $C_1$:

\begin{align*}
    C_1 =& \expect\!\left[\suml\sumh (\truegaprs-\Dpirs)\indicator\left\{{\cV_{h,\ell}(A_{h,\ell})}\right\}\indicator\left\{\cE_{h,\ell}\right\}\right]\\
    =& \expect\!\left[\suml\sumh (\truegaprs-\Dpirs)\indicator\left\{\truegaprs > \Dpirs+\epsilon\mu_0\right\}\indicator\left\{\cE_{h,\ell}\right\}\right]
\end{align*}
As we assign the probability of pulling $A_{h,\ell}$ from $B_{h,\ell}$ as exactly $\frac{1}{2}$, upper bound analysis of term $C_1$ closely matches with the analysis for term $A_2$ in~\citep{balagopalan2024minimum}. Thus, we have--

$$C_1 \leq 6(B-\Dpirs)d\log\left(1+\frac{2}{\lambda} \right) \leq \frac{6(B-\Dpi)d}{\sigma_{\max}}\log\left(1+\frac{2}{\lambda} \right)$$

\textit{Upper bound on $C_2$.} We further decompose $C_2$ based on the event based on the event $\cF_{h,\ell}(A_{h,\ell})$--

\begin{align*}
    C_2 =& \expect\!\left[\suml\sumh (\truegaprs-\Dpirs\indicator\left\{{\cV_{h,\ell}(A_{h,\ell})}\right\}\indicator\left\{\overline{\cE_{h,\ell}}\right\}\right]\\
    \leq & \underbrace{\expect\!\left[\suml\sumh (\truegaprs-\Dpirs)\indicator\left\{{\cV_{h,\ell}(A_{h,\ell})}\right\}\indicator\left\{\overline{\cE_{h,\ell}}\right\}\indicator\left\{ \cF_{h,\ell}(A_{h,\ell}) \right\}\right] }_{\defn D_1}\\
    &\qquad\qquad+ \underbrace{\expect\!\left[\suml\sumh (\truegaprs-\Dpirs)\indicator\left\{{\cV_{h,\ell}(A_{h,\ell})}\right\}\indicator\left\{\overline{\cE_{h,\ell}}\right\}\indicator\left\{ \overline{\cF_{h,\ell}(A_{h,\ell})} \right\}\right]}_{\defn D_2}
\end{align*}

\textit{Upper bound on $D_1$.} Now, we focus on the term $D_1$--

\begin{align*}
    D_1 =& \expect\!\left[\suml\sumh (\truegaprs-\Dpirs)\indicator\left\{{\cV_{h,\ell}(A_{h,\ell})}\right\}\indicator\left\{\overline{\cE_{h,\ell}}\right\}\indicator\left\{ \cF_{h,\ell}(A_{h,\ell}) \right\}\right] 
\end{align*}

To upper bounding this term, we follow the upper bound formalisation of Term $D_2$ in~\citep{balagopalan2024minimum} without the peeling argument to get--

\begin{align*}
    D_1 \leq& 3d\frac{(B-\Dpirs)}{\varepsilon_2}\log\left( 1 + \frac{2}{\lambda\varepsilon_2}\right) + 3(B-\Dpirs)d\log\left( 1 + \frac{2}{\lambda}\right)\\
    \leq& 3d\frac{(B-\Dpi)}{\sigma_{\max}\varepsilon_2}\log\left( 1 + \frac{2}{\lambda\varepsilon_2}\right) + \frac{3(B-\Dpi)d}{\sigma_{\max}}\log\left( 1 + \frac{2}{\lambda}\right)
\end{align*}

\textit{Upper bound on $D_2$.} To upper bound this term, we again decompose it by leveraging the event $\cU_{h,\ell}(A_{h,\ell})$--

\begin{align*}
    D_2 =& \expect\!\left[\suml\sumh (\truegaprs-\Dpirs)\indicator\left\{{\cV_{h,\ell}(A_{h,\ell})}\right\}\indicator\left\{\overline{\cE_{h,\ell}}\right\}\indicator\left\{ \overline{\cF_{h,\ell}(A_{h,\ell})} \right\}\right]\\
    \leq & \underbrace{\expect\!\left[\suml\sumh (\truegaprs-\Dpirs)\indicator\left\{{\cV_{h,\ell}(A_{h,\ell})}\right\}\indicator\left\{\overline{\cE_{h,\ell}}\right\}\indicator\left\{ \overline{\cF_{h,\ell}(A_{h,\ell})} \right\}\indicator\left\{ \cU_{h,\ell}(A_{h,\ell})\right\}\right] }_{\defn E_1}\\
    &\qquad\qquad \underbrace{\expect\!\left[\suml\sumh (\truegaprs-\Dpirs)\indicator\left\{{\cV_{h,\ell}(A_{h,\ell})}\right\}\indicator\left\{\overline{\cE_{h,\ell}}\right\}\indicator\left\{ \overline{\cF_{h,\ell}(A_{h,\ell})} \right\}\indicator\left\{ \overline{\cU_{h,\ell}(A_{h,\ell})}\right\}\right] }_{\defn E_2}
\end{align*}

\textit{Upper bound on $E_1$.} This upper bound can be obtained following the steps in~\citep{balagopalan2024minimum} for their term $F_1$ ignoring the peeling steps. Thus, we have--

\begin{align*}
    E_1 \leq& \frac{1}{\alpha_{\rm emp}}(B-\Dpirs)HL\exp\left(-\frac{(\Dpirs+\epsilon\mu_0)^2}{4\varepsilon_2 \beta_{h,\ell}(\alpha_{h,\ell})(1+c)^2} \right)\\
    &\qquad\qquad+\frac{3d(B-\Dpirs)}{\alpha_{\rm emp}^2}\left( \frac{1}{\varepsilon_2}\log\left( 1 + \frac{2}{\lambda\varepsilon_2}\right) + \log\left( 1 + \frac{2}{\lambda}\right)\right)\\
    \leq&\frac{1}{\sigma_{\max}\alpha_{\rm emp}}(B-\Dpi)HL\exp\left(-\frac{(\Dpi+\epsilon\mu_0)^2}{4\sigma_{\max}^2 \varepsilon_2 \beta_{h,\ell}(\alpha_{h,\ell})(1+c)^2} \right)\\
    &+\frac{3d(B-\Dpi)}{\sigma_{\max}\alpha_{\rm emp}^2}\left( \frac{1}{\varepsilon_2}\log\left( 1 + \frac{2}{\lambda\varepsilon_2}\right) + \log\left( 1 + \frac{2}{\lambda}\right)\right)
\end{align*}

\textit{Upper bound on $E_2$.} We finally use the event $\cJ_{h,\ell}(h,\ell)$ to decompose the event $E_2$ as--
\begin{align*}
    &E_2 =  \expect\!\left[\suml\sumh (\truegaprs-\Dpirs)\indicator\left\{{\cV_{h,\ell}(A_{h,\ell})}\right\}\indicator\left\{\overline{\cE_{h,\ell}}\right\}\indicator\left\{ \overline{\cF_{h,\ell}(A_{h,\ell})} \right\}\indicator\left\{ \overline{\cU_{h,\ell}(A_{h,\ell})}\right\}\right]\\
    \leq & \underbrace{\expect\!\left[\suml\sumh (\truegaprs-\Dpirs)\indicator\left\{{\cV_{h,\ell}(A_{h,\ell})}\right\}\indicator\left\{\overline{\cE_{h,\ell}}\right\}\indicator\left\{ \overline{\cF_{h,\ell}(A_{h,\ell})} \right\}\indicator\left\{ \overline{\cU_{h,\ell}(A_{h,\ell})}\right\}\indicator\left\{ \cJ_{h,\ell}(A_{h,\ell})\right\}\right]}_{\defn F_1}\\
    &+\underbrace{\expect\!\left[\suml\sumh (\truegaprs-\Dpirs)\indicator\left\{{\cV_{h,\ell}(A_{h,\ell})}\right\}\indicator\left\{\overline{\cE_{h,\ell}}\right\}\indicator\left\{ \overline{\cF_{h,\ell}(A_{h,\ell})} \right\}\indicator\left\{ \overline{\cU_{h,\ell}(A_{h,\ell})}\right\}\indicator\left\{ \overline{\cJ_{h,\ell}(A_{h,\ell})}\right\}\right]}_{\defn F_2}
\end{align*}

\textit{Upper bound on $F_1$.} It is easy to see that in $F_1$, if the good event in Equation~\eqref{eq:good_event_hl} does not happen, then under $\overline{\cG_1}$, 
\[F_1\mid_{\overline{\cG_1}}~ \leq (B-\Dpirs)\sumh\suml {\alpha_{h,\ell}}\leq \leq \frac{(B-\Dpi)}{\sigma_{\max}}\sumh\suml {\alpha_{h,\ell}}.\] 

On the other hand, under the event $\cG_1$, following~\citep{balagopalan2024minimum}, we write
\begin{align*}
    F_1 \mid_{\cG_1}\leq& \expect\left[ \suml\sumh (\truegaprs-\Dpirs) \indicator\left\{ (\thetahat_{h,\ell}-\theta)^{\top}\feature{A_{h,\ell}} > \frac{c(\Dpirs+\epsilon\mu_0)}{1+c}-\varepsilon_1 \right\} \indicator\{\cG_1\}\right]\\
    \leq&\expect\left[ \suml\sumh (\truegaprs-\Dpirs) \indicator\left\{ \|\phi_{A_{h,\ell},h,\ell}\|_{{(\gram{h}{\ell})}^{-1}}^2\geq\frac{1}{\beta_{h,\ell}(\alpha_{h,\ell})}\left(\frac{c(\Dpirs+\epsilon\mu_0)}{2(1+c)}\right)^2  \right\} \indicator\{\cG_1\}\right]\\
    \stackrel{(a)}{\leq }&\expect\left[ \suml\sumh (\truegaprs-\Dpirs) \indicator\left\{ \|\phi_{A_{h,\ell},h,\ell}\|_{{(\gram{h}{\ell})}^{-1}}^2\geq\frac{1}{\lambda S_*^2}\left(\frac{c(\Dpirs+\epsilon\mu_0)}{2(1+c)}\right)^2  \right\} \indicator\{\cG_1\}\right]\\
    \leq& 3(B-\Dpirs)d\log\left( 1 + \frac{2}{\lambda} \right) \leq \frac{3(B-\Dpi)d}{\sigma_{\max}}\log\left( 1 + \frac{2}{\lambda} \right) \,,
\end{align*}
choosing $\varepsilon_1 = \frac{c(\Dpirs+\epsilon\mu_0)}{2(1+c)}$.

Thus, final upper on $F_1$ is given by,
\begin{align*}
    F_1 \leq \frac{3(B-\Dpi)d}{\sigma_{\max}}\log\left( 1 + \frac{2}{\lambda} \right)+ \frac{(B-\Dpi)}{\sigma_{\max}}\sumh\suml {\alpha_{h,\ell}}
\end{align*}

\textit{Bound on $F_2$.} We follow the analysis of term $F_3$ of~\citep{balagopalan2024minimum} to state the final upper bound on $F_2$ with $H_{\max}=e$ in our case--

\begin{align*}
    F_2 \leq& \frac{128 e(B-\Dpi) C_{\rm opt}d\log d}{\sigma_{\max}\alpha_{\rm opt}\varepsilon_1^2}\left(\frac{\lambda S_*^2}{2}+d\log\left(1+\frac{\sumh\suml \sigma_{h,\ell}^2}{d\lambda}\right)\right)
\end{align*}

Hence, putting all the components in Equation~\eqref{eq:constr-decomp}, we get the final bound on $\Reg^{\rm constr}$ as:
\begin{align*}
    \Reg^{\rm constr}
    \leq& \frac{128 e(B-\Dpi) C_{\rm opt}d\log d}{\alpha_{\rm opt}\varepsilon_1^2}\left(\frac{\lambda S_*^2}{2}+d\log\left(1+\frac{\sumh\suml \sigma_{h,\ell}^2}{d\lambda}\right)\right)\\
    &+12(B-\Dpi)d\log\left( 1 + \frac{2}{\lambda} \right) + (B-\Dpi)\sumh\suml {\alpha_{h,\ell}}\\
    &+\frac{1}{\alpha_{\rm emp}}(B-\Dpi)HL\exp\left(-\frac{(\Dpi+\epsilon\mu_0)^2}{4\sigma_{\min}^2\varepsilon_2 \beta_{h,\ell}(\alpha_{h,\ell})(1+c)^2} \right)\\
    &+\frac{3d(B-\Dpi)}{\alpha_{\rm emp}^2}\left( \frac{1}{\varepsilon_2}\log\left( 1 + \frac{2}{\lambda\varepsilon_2}\right) + \log\left( 1 + \frac{2}{\lambda}\right)\right)\\
    &+3d\frac{(B-\Dpi)}{\varepsilon_2}\log\left( 1 + \frac{2}{\lambda\varepsilon_2}\right)  \\
    &+ 12\Dpi d\log\left(1+\frac{2}{\lambda}\right) + 2\Dpi\suml\sumh \alpha_{h,\ell}\\
    &=  \tilde{\bigO}\left( \frac{B-\Dpi}{(\Dpi+\epsilon\mu_0)^2} d^2\left(\log^2(LH) + \log\left(\sumh\suml \sigma_{h,\ell}^2\right) \right)\right)
\end{align*}
where the last inequality holds due to putting $\varepsilon_2 = \frac{(\Dpi+\epsilon\mu_0)^2}{4\sigma_{\max}^2\beta_{h,\ell}(\alpha_{h,\ell})\log(LH)}$, $c=1$, and $\alpha_{h,\ell} = \frac{1}{LH}$. Recall, $\lambda  = \frac{2S_*^2}{\sigma_{\max}\min\left\{\left( \frac{\Dpi+\epsilon\mu_0}{2} - B\right)^2 , \left( {\Dpi+\epsilon\mu_0} - \frac{B}{2}\right)^2 \right\}}$.

\begin{tcolorbox}[top=1pt,bottom=1pt,left=2pt,right=2pt]
Thus, Algorithm~\ref{alg:episodic-pref-imed-full} and~\ref{alg:med_step} jointly satisfies --
    $$\Reg^{\rm constr} =  \tilde{\bigO}\left( \frac{\sigma_{\max}^2(B-\Dpi)}{(\Dpi+\epsilon\mu_0)^2} d^2\left(\log^2(LH) + \log\left(\sumh\suml \sigma_{h,\ell}^2\right) \right)\right).$$
\end{tcolorbox}

\subsection{Part II: Constraint Satisfying Regret Upper Bound}
\begin{reptheorem}{thm:opt_reg_id}
Thus, Algorithm~\ref{alg:episodic-pref-imed-full} and~\ref{alg:med_step} jointly satisfies --

    \begin{align*}
        \Reg^{\rm safe} = \tilde{\bigO}\left( \frac{d^2\left(\log^2(LH)+\log\left( \suml {(\discount)}^{L-\ell} \sumh \sigma^2_{h,\ell}\right)\right)}{\Delta_{\min}} \right)
    \end{align*}
\end{reptheorem}

\begin{proof}

In this section, we provide a complete proof of the regret upper bound under constraint satisfaction, i.e, upper bound on $\Reg^{\rm opt}$. The proof structure involves two steps:

\begin{enumerate}
    \item Then for the upper bound, we first define the conditioning event that will be used to decompose the regret expression further down the line.
    \item We bound each part of the decomposition step by step to get a final accumulated upper bound.
\end{enumerate}

\subsubsection{Step 1: Defining the Conditioning Events}
To prove upper bound on $\Reg^{\rm safe}$, we recall the events defined in Part I (Definition~\ref{def:cond_events}) with additional peeling events.
\begin{align*}
        &\cU_{h,\ell}(a) \defn \left\{ \estgaprs \ge \frac{\truegaprs}{1+c} \right\} \quad\quad \text{for some } c\geq 0 \notag\\
        & \cE_{h,\ell} \defn \left\{ |B_{h,\ell}|>0\right\}, \quad \text{where } B_{h,\ell} = \bigl\{a:\|\phi_{a,h,\ell}\|^{2}_{(V_{h,\ell}^{\gamma_{\rm decay}})^{-1}}>1\bigr\}\notag\\
        & \cC_{h,\ell,k}(A_{h,\ell}) \defn \left\{\|\phi_{a,h,\ell}\|_{(\gram{h}{\ell})^{-1}} > \varepsilon_k \right\}\\
        & \cD_{h,\ell,k}(A_{h,\ell}) \defn \left\{ (\Dpirs+\epsilon\mu_0)2^{-k} < \bar{\Delta}_{A_{h,\ell},h,\ell} < (\Dpirs+\epsilon\mu_0)2^{-k+1} \right\} \\
        &\overline{\cD_{h,\ell,K}(A_{h,\ell})} \defn \left\{ \bar{\Delta}_{A_{h,\ell}} \leq (\Dpirs+\epsilon\mu_0)2^{-K} \right\}\\
        &\cJ_{h,\ell}(a) \defn \left\{ \theta^{\top}\phi_{a_{h,\ell}^*,h,\ell} - \thetahat_{h,\ell}^{\top}\phi_{a_{h,\ell}^*,h,\ell} \leq \varepsilon_{2,k} \right\}\,,
\end{align*}

where we define $\bar{\Delta}_{A_{h,\ell},h,\ell} \defn \max\{\truegaprs,\Dpirs+\epsilon\mu_0\}$.

\subsubsection{Decomposition of $\Reg^{\rm safe}$}

We follow the proof structure from~\citep{balagopalan2024minimum} for this part. We use similar peeling technique and decomposition of regret using conditioning events. As we have constraint satisficing in this case, we use $(\Dpirs+\epsilon\mu_0)$ as the upper bound for $\truegaprs$.

\textit{Final Regret Bound under Constraint Satisficing.}
\begin{align*}
        \frac{1}{\sigma_{\max}}\Reg^{\rm safe} \leq& 18(\sigma_{\max}\Dpi+\epsilon\mu_0)d\log\left(1+\frac{2}{\lambda}\right)\\
        &+ LH(\sigma_{\max}\Dpi+\epsilon\mu_0)2^{-P} \indicator\left\{ (\sigma_{\max}\Dpi+\epsilon\mu_0)2^{-P}>\Delta_{\min}\right\}\\
        &+\frac{12d(\sigma_{\max}\Dpi+\epsilon\mu_0)}{2^{-P}\varepsilon}\log\left(1+\frac{2}{\lambda2^{-2P }\varepsilon}\right)\\
        &+\frac{1}{\alpha_{\rm emp}}4LH(\sigma_{\max}\Dpi+\epsilon\mu_0)\exp\left( -\frac{(\sigma_{\max}\Dpi+\epsilon\mu_0)^2}{16\varepsilon\beta_{h,\ell}(\alpha_{h,\ell})}\right)\\
        &+ \frac{1}{\alpha_{\rm emp}^2}\Bigg( \frac{12d(\sigma_{\max}\Dpi+\epsilon\mu_0)}{2^{-P}\varepsilon}\log\left(1+\frac{2}{\lambda2^{-2P}\varepsilon}\right)\\
        &+6(\sigma_{\max}\Dpi+\epsilon\mu_0)d\log\left(1+\frac{2}{\lambda}\right)\Bigg)\\
        &+ 2(\sigma_{\max}\Dpi+\epsilon\mu_0)\log(LH)\\
        &+ \frac{192\beta_{h,\ell}(\alpha_{h,\ell})d}{(\sigma_{\min}\Dpi+\epsilon\mu_0)2^{-P}}\log\left( 1+\frac{32\beta_{h,\ell}(\alpha_{h,\ell}}{\lambda(\sigma_{\min}\Dpi+\epsilon\mu_0)^2 2^{-2P}}\right)\\
        &+ \frac{512d\log(d) e C_{\rm opt}}{\alpha_{\rm opt}(\sigma_{\min}\Dpi+\epsilon\mu_0)2^{-P}}\left( \frac{\lambda S_*}{2}+d\log\left( 1 +\frac{\suml\sumh \sigma_{a,h,\ell}^2}{d\lambda} \right)\right)
    \end{align*}
\end{proof}

Putting $\epsilon = \frac{(\sigma_{\min}\Dpi+\epsilon\mu_0)^2}{16\beta_{h,\ell}(\alpha_{h,\ell})}$

\begin{tcolorbox}[top=1pt,bottom=1pt,left=2pt,right=2pt]
Algorithm~\ref{alg:episodic-pref-imed-full} and~\ref{alg:med_step} jointly satisfy--
\begin{align*}
    \Reg^{\rm safe} = \tilde{\bigO}\left( \kappa\frac{d^2\left(\log^2(LH)+\log\left( \sum\sumh \sigma^2_{a,h,\ell}\right)\right)}{\Delta_{\min}} \right) 
\end{align*}
    where $\kappa \defn \frac{\max_{a,h,\ell}\sigma_{a,h,\ell}^2}{\min_{a,h,\ell}\sigma_{a,h,\ell}^2}$.
\end{tcolorbox}

\section{Upper Bound on Expected Constraint Violation}

In this section, we prove the upper bound guaranty on the expected number of constraint violation of \framework.

\begin{reptheorem}{thm:cont_viol}

\framework~ enjoys  
    \begin{align*}
        \text{Violation}(\bpi_0) = \tilde{\bigO}\left( d\right)\,,
    \end{align*}
\end{reptheorem}

\begin{proof}
    We start from the definition of expected constraint violation--
    \begin{align*}
        \text{Violation}(\bpi_0) = \expect\left[ \indicator\left\{ \cV_{h,\ell}(A_{h,\ell}) \right\} \right]
    \end{align*}
From the standard regret analysis in Section~\ref{supp:regret_ub}, we know $\text{Violation}(\bpi_0) = A_1$. 

Thus, setting $\lambda = \frac{2S_*^2}{\min\left\{\left( \frac{\Dpi+\epsilon\mu_0}{2} - B\right)^2 , \left( {\Dpi+\epsilon\mu_0} - \frac{B}{2}\right)^2 \right\}}$ to get the final bound on $B_2$ as:
\begin{align*}
     \text{Violation}(\bpi_0)  \leq   12d\log\left( 1+\frac{\min\left\{\left( \frac{\Dpi+\epsilon\mu_0}{2} - B\right)^2 , \left( {\Dpi+\epsilon\mu_0} - \frac{B}{2}\right)^2 \right\}}{S_*}\right) + 2 = \tilde{\bigO}(d)\,.
\end{align*}
\end{proof}\clearpage
\section{Discussion on the Lagrangian Dual Variable}

\paragraph{Lyapunov Drift for the Dual Variable.} Recall the dual update with $\rho \defn \frac{\epsilon}{(1-\epsilon)}$ in Algorithm~\ref{alg:episodic-pref-imed-full}:
\begin{equation}\label{eq:dual-update}
  \nu_{h+1,\ell} = \max\!\bigl\{0,\;\nu_{h,\ell} + \eta_{h}(\xi_{h,\ell} - \rho)\bigr\},
  \qquad
  \xi_{h,\ell} := \indicator\left[\inner{\hat\theta_{h,\ell}}{\phi_{A_{h,\ell},h}} < \tau_\ell\right],
\end{equation}
with $\eta_{h,\ell} = \eta/\sqrt{h}$, $\eta > 0$ to be chosen.
Note, $\xi_{h,\ell} \in \{0,1\}$ regardless of the reward distribution.
Thus, the dual update depends only on whether the \emph{estimated} reward crosses the threshold,
not on the magnitude of the reward itself. 

\begin{lemma}[Lyapunov drift inequality]\label{lem:lyapunov}
  For all $t \ge 1$:
  \begin{align*}
    \E[\nu_{h+1,\ell}^2 - \nu_{h,\ell}^2 \mid \cF_t]
    \;\le\; 2\eta_{h,\ell}\bigl(\E[\xi_{h,\ell} \mid \cF_{h,\ell}] - \rho\bigr)\nu_{h,\ell} + \eta_{h,\ell}^2.
  \end{align*}
\end{lemma}
\begin{proof}
  Using $\max\{0,x\}^2 \le x^2$:
  \begin{align*}
    \nu_{h+1,\ell}^2 \le \bigl(\nu_{h,\ell} + \eta_{h,\ell}(\xi_{h,\ell} - \rho)\bigr)^2
    = \nu_{h,\ell}^2 + 2\eta_{h,\ell}(\xi_{h,\ell} - \rho)\nu_{h,\ell} + \eta_{h,\ell}^2(\xi_{h,\ell} - \rho)^2.
  \end{align*}
  Since $\xi_{h,\ell} \in \{0,1\}$, we have $(\xi_{h,\ell} - \rho)^2 \le 1$.
  Taking conditional expectation gives the result.
\end{proof}

\begin{lemma}[Dual variable growth]\label{lem:nu-bound}
  With $\nu_0 = 0$ and $\eta_{h,\ell} = \eta/\sqrt{\ell h}$, we have $\nu_{h,\ell} \le 2\eta\sqrt{\ell h}$.
\end{lemma}
\begin{proof}
  Since $\xi_{h,\ell} \le 1$ we have $\nu_{h+1,\ell} \le \nu_{h,\ell} + \eta_{h,\ell}$, so
  $\nu_{h,\ell} \le \suml\sumh\eta/\sqrt{s} \le 2\eta\sqrt{LH}$.
\end{proof}
The dual update in Equation~\eqref{eq:dual-update} is exactly projected online gradient descent
(OGD) on the sequence of losses $\ell_{h,\ell}(\nu) \defn  -\nu(\xi_{h,\ell} - \rho)$ over $\nu\in[0,\infty)$.
The gradient at $\nu_{h,\ell}$ is $-(\xi_{h,\ell}-\rho)$, which lies in $(-1, 1]$ since $\xi_{h,\ell}\in\{0,1\}$
and $\rho\in[0,1)$. 

\begin{lemma}[OGD regret bound]\label{lem:ogd}
  For any $\nu^* \ge 0$:
  \begin{align*}
    \suml\sumh \cL_t(\nu_{h,\ell}) - \suml\sumh \cL_t(\nu^*)
    \;\le\; \frac{{\nu^*}^2}{2\eta_{LH}} + \eta\sqrt{LH} .
  \end{align*}
\end{lemma}
\begin{proof}
  The loss $\cL_{h,\ell}(\nu) = -\nu(\xi_{h,\ell} - \rho)$ is linear in $\nu$, so it is
  convex, and its gradient is $\nabla_t = \nabla_\nu \cL_{h,\ell}(\nu_{h,\ell}) = -(\xi_{h,\ell} - \rho)$.
  Since $\xi_{h,\ell} \in \{0,1\}$ and $\rho \in [0,1)$, we have $\nabla_{h,\ell} \in [-(1-\rho), \rho]
  \subset (-1,1]$, so $\nabla_{h,\ell}^2 \le 1$. We use~\cite[Theorem 3.1]{hazan2016introduction} with $\nu_{1,1} = 0$ to directly get the desired result.
\end{proof}
\clearpage
\section{Dri-IMED: Drift Adaptive Indexed Minimum Empirical Divergence}\label{supp:dri-imed}

For empirical evaluation, in this section we provide the implemented pseudocode of the \texttt{Dri-IMED} algorithm, that is an Indexed Minimum Empirical Divergence version of MED strategy used in this paper. The main difference from Algorithm~\ref{alg:med_step} is in Line 4 of \texttt{Dri-IMED}. It computes the Lagrangian dual-augmented IMED index being optimistic about the action, but pessimistic with respect to the constraint. Whenever a pulled arm violates the constraint the IMED index is penalised. At the end it plays the arm that has the minimum index.   

\begin{algorithm}[h]
    \caption{\texttt{Dri-IMED}: \textbf{Dri}ft adaptive \textbf{I}ndexed \textbf{M}inimum \textbf{E}mpirical \textbf{D}ivergence (IMED) Step}
    \label{alg:imed_step}

\begin{algorithmic}[1]
    \REQUIRE Episode $\ell \in [L]$, user id $h\in[H]$, user's rescaled features $\left\{\phi_{a,h,\ell}\right\}_{a=1}^{|\cA|}$, $\thetahat_{h,\ell}, V_{h,\ell}^{\gamma_{\rm decay}},\lambda$.
    \REQUIRE Constraint threshold $\tau_\ell$, Lagrangian multiplier $\nu_\ell$.
    \STATE  \textcolor{blue}{\textbf{Compute confidence radius and upper and lower confidence bound:}} For each arm $a\in\cA$,
     \begin{align*}
        \beta_{h,\ell}(\alpha_{h,\ell})  &= \left( \sqrt{\log\frac{\det V_{h,\ell}^{\gamma_{\rm decay}}}{\det V_0}} + 2\log\frac{1}{\alpha_{h,\ell}} + \sqrt{\lambda} S\right)^2\\
        \text{UCB}(a,h,\ell) &= \langle \hat{\theta}_{h,\ell}, \phi_{a,h,\ell} \rangle + \sqrt{\beta_{h,\ell}(\alpha_{h,\ell})} \|\phi_{a,h,\ell}\|_{(V_{h,\ell}^{\gamma_{\rm decay}})^{-1}} \\
        \text{LCB}(a,h,\ell) &= \langle \hat{\theta}_{h,\ell}, \phi_{a,h,\ell} \rangle - \sqrt{\beta_{h,\ell}(\alpha_{h,\ell})} \|\phi_{a,h,\ell}\|_{(V_{h,\ell}^{\gamma_{\rm decay}})^{-1}}
        \end{align*}
    \STATE Empirical best action: $\hat{a}_{h,\ell} = \argmax_{a\in\cA} \langle \thetahat_{h,\ell}, \phi_{a,h,\ell} \rangle$
    \STATE Maximum empirical reward: $\hat{\mu}_{\max}(h,\ell) = \max_{a\in\cA} \langle \hat{\theta}_{h,\ell}, \phi_{a,h,\ell} \rangle$
    \STATE \textcolor{blue}{\textbf{Lagrangian-augmented IMED index:}} For each arm $ a \in \cA$:
        \begin{align*}
            I_{h,\ell}^{\text{aug}}(a) = \begin{cases}
                \nu_\ell \cdot \max\{0, \tau_\ell - \text{LCB}(a,h,\ell)\} & \text{if } a = \hat{a}_{h,\ell} \\
                N_{h,\ell}(a) \cdot \frac{(\text{UCB}(a,h,\ell) - \hat{\mu}_{\max}(h,\ell))^2}{2\sigma_{h,\ell}^2} & \\
                \quad + \nu_\ell \cdot \max\{0, \tau_\ell - \text{LCB}(a,h,\ell)\} & \text{if UCB}(a,h,\ell) \geq \mu_{\max}(h,\ell) 
            \end{cases}
        \end{align*}
    \STATE Pull $A_{h,\ell} = \argmin_{a \in \cA} I_{h,\ell}^{\text{aug}}(a) $
\end{algorithmic}
\end{algorithm}\clearpage
\section{Experimental Analysis}

\label{app:experiments}

To assess the performance of \texttt{Dri-MED} and \texttt{Dri-IMED}, we conduct 
numerical experiments on a synthetic episodic contextual linear bandit with 
preference feedback and context drift. The environment consists of $H$ users, $A$ arms, and $L$ 
episodes, with a $d$-dimensional parameter $\theta^*$ and $M$-dimensional feedback 
signals. The mean feedback signal $\mathbb{E}[Y_{h,\ell} \mid a] = \Phi(a)\theta^*$ 
is context-independent, while the covariance 
scales with the observed context norm, inducing heteroscedastic noise. User 
preference vectors $\{p_h\}$ are drawn from a Dirichlet distribution and held 
fixed across episodes. 

\paragraph{Setup.} The full environment generation is detailed in Algorithm~\ref{alg:env-gen}.
The baseline policy $\pi_0$ is set as the arm at the median performance quantile
for each user, i.e.\ $\pi_0(h) = \arg\min_a |\phi_{a,h}^\top\theta^\star - q_{0.5}^h|$
where $q_{0.5}^h$ is the median of $\{\phi_{a,h}^\top\theta^\star\}_{a \in [A]}$.
We use the following environment parameters throughout all experiments:
feature dimension $d = 4$, feedback dimension $M = 4$, number of arms $A = 5$,
number of users $H = 10$, number of episodes $L = 1000$, parameter norm $B = 1$,
preference scale $s = 2$, context dimension $D = 2$, context noise $\sigma_C = 1$,
and reward noise $\sigma_r = 0.1$. Results that are presented in Figure~\ref{fig:results} are averaged over $128$ independent seeds.

\begin{algorithm}[htbp]
\caption{Environment Generation}
\begin{algorithmic}[1]
\REQUIRE $d, M, A, H, L, f, \sigma_C, \sigma_r, s, B$
\STATE Draw $\theta^\star \sim \mathcal{N}(0, I_d)$, normalise to $\|\theta^\star\| = B$
\STATE Draw row-normalised $\Phi(a) \in \mathbb{R}^{M \times d}$ and $\Sigma_a = S_a S_a^\top / M$ for each $a \in [A]$
\STATE Draw $\omega_h \sim \mathrm{Dir}(s \cdot p_h)$ with $p_h \sim \mathrm{Dir}(\mathbf{1}_M)$ for each $h \in [H]$
\STATE Set $\kappa_\ell \leftarrow 1 + f(\ell, L)$ for $\ell \in [L]$
\FOR{$\ell = 1$ \TO $L$}
    \STATE Sample $C_{h,\ell} \sim \mathcal{N}(0,\, \sigma_C^2 \kappa_\ell^2\, I_D)$ for each $h \in [H]$
    \STATE Set $\phi_{a,h} \leftarrow \Phi(a)^\top \omega_h$ for all $a \in [A],\, h \in [H]$
    \STATE Draw $Y_{h,\ell} \sim \mathcal{N}\!\left(\Phi(A_{h,\ell})\theta^\star,\; \sigma_r(1+\|C_{h,\ell}\|)\,\Sigma_{A_{h,\ell}}\right)$, set $r_{h,\ell} \leftarrow \omega_h^\top Y_{h,\ell}$
\ENDFOR
\end{algorithmic}
\label{alg:env-gen}
\end{algorithm}

\paragraph{Drifting.} We consider four drift regimes for the context distribution, parameterized by a 
magnitude $\kappa > 0$ and a shared drift direction $v \in \mathbb{R}^C$, 
$\|v\|=1$, so that $\mu_\ell = f(\ell) \cdot v$:
\begin{itemize}
    \item \emph{No drift}: $f(\ell) = 0$,
    \item \emph{Gradual drift}: $f(\ell) = \kappa \cdot \ell / L$,
    \item \emph{Periodic drift}: $f(\ell) = \kappa \cdot \dfrac{1}{2}\left(1 + 
    \sin\!\left(\dfrac{2\pi \ell}{L/50}\right)\right)$,
    \item \emph{Abrupt drift}: $f(\ell) = \kappa \cdot \sum_{i} i \cdot 
    \mathbf{1}\{\ell \geq \tau_i\}$, where $\tau_1 = \lfloor L/3 \rfloor$ and 
    $\tau_2 = \lfloor 2L/3 \rfloor$ are fixed change-points.
\end{itemize}

\paragraph{Baselines.} To the best of our knowledge, no prior algorithm addresses this combined setting of preference structure, heteroscedastic 
noise, and non-stationary context distributions;
Since no existing algorithm is designed for this setting, we evaluate 
\texttt{Dri-MED} and \texttt{Dri-IMED} against four stationary linear bandit 
baselines that ignore both the drift and the preference structure: \texttt{OFUL} 
\citep{abbasi2011improved}, \texttt{LinMED} \citep{balagopalan2024minimum}, 
\texttt{LinIMED} \citep{bian2024indexed}, and \texttt{LinTS} 
\citep{agrawal2013thompson}. 
For all drift types, we set magnitude $\kappa = 100$, so the maximum context scale is
$s_{\max} = 1 + \kappa = 101$.
The per-arm noise standard deviation is
$\sigma_{a,h,\ell} = \sqrt{\omega_h^\top \Sigma_a \cdot \sigma_r \cdot (1 + \|C_{h,\ell}\|) \cdot \omega_h}$,
which is bounded by
\[
    \bar\sigma
    = \sqrt{\sigma_r \cdot (1 + \sigma_c \cdot s_{\max} \cdot \sqrt{D}) \cdot \max_{a,i} [\Sigma_a]_{ii}}
\]
where $\sigma_r = 0.1$ is the reward noise, $\sigma_c = 1.0$ the context noise scale,
$D = 2$ the context dimension, and $\max_{a,i}[\Sigma_a]_{ii}$ the largest diagonal entry
across all arm covariance matrices (evaluated once at initialisation).
In our experiments this yields $\bar\sigma \approx 1.92$, which we use as the
sub-Gaussianity parameter for all algorithms.
We report cumulative regret and cumulative 
satisficing constraint violations, averaged over independent random seeds. All specific hyperparameters used by the algorithm are summarized in Table \ref{tab:hyperparams}

\begin{table}[h]
\centering
\caption{Hyperparameters used per algorithm. Parameters shared by all algorithms:
regularisation $\lambda = \bar\sigma^2 / S^2$, confidence $\delta = 0.01$, $S = 1$.}
\label{tab:hyperparams}
\small
\begin{tabular}{lcccccc}
\toprule
\textbf{Algorithm} & $\gamma_{\text{decay}}$ & $\varepsilon$ & $\nu_0$ & $\alpha_{\text{opt}}$ & $\alpha_{\text{emp}}$ & $C$ \\
\midrule
\texttt{OFUL}     & --     & --    & --    & --      & --      & -- \\
\texttt{LinTS}    & --     & --    & --    & --      & --      & -- \\
\texttt{LinMED}   & --     & --    & --    & $0.99$  & $0.005$ & -- \\
\texttt{LinIMED}  & --     & --    & --    & --      & --      & $30$ \\
\texttt{Dri-MED}  & $0.99$ & $0.1$ & $1.0$ & $0.99$  & $0.005$ & -- \\
\texttt{Dri-IMED} & $0.99$ & $0.1$ & $1.0$ & --      & --      & -- \\
\bottomrule
\end{tabular}
\end{table}

\paragraph{Results.} Figure~\ref{fig:results} report cumulatieve regret and constraint violation rates across all drift regimes. \texttt{Dri-MED} and and \texttt{Dri-IMED} consistently 
achieve very low cumulative regret across all settings, outperforming all 
stationary baselines by a large margin. Among baselines, \texttt{OFUL} performs 
best but still incurs regret an order of magnitude larger than our methods; 
\texttt{LinMED} and \texttt{LinTS} perform worst, highlighting the cost of 
ignoring the preference and heteroscedastic structure. Importantly, both 
\texttt{Dri-MED} and \texttt{Dri-IMED} maintain near-zero true constraint 
violation throughout all episodes, whereas baselines show persistent violations. 
The results are stable across drift regimes: even under abrupt drift with 
magnitude $\kappa = 100$, the performance of our methods is largely unaffected, 
confirming that the drift-adaptive design successfully absorbs the non-stationarity 
through the discounted regression and context-scaled normalisation. 
Figure~\ref{fig:arm_allocation} further illustrates that \texttt{Dri-MED} 
concentrates pulls on the oracle-optimal arm for every user, while \texttt{LinMED} 
spreads mass across suboptimal arms, and $\pi_0$ is always fixed at the median arm.

\paragraph{Ablations.} Figure~\ref{fig:ablation} examines the sensitivity of \texttt{Dri-MED} and 
\texttt{Dri-IMED} to the baseline quantile $q$ and satisficing tolerance 
$\varepsilon$ under abrupt drift. As $q$ increases, the baseline policy 
$\pi_0$ becomes stronger, tightening the constraint and reducing the feasible 
arm set; this forces the algorithm to focus on higher-quality arms earlier, 
yielding lower final regret. Conversely, increasing $\varepsilon$ loosens the 
constraint, allowing more exploration of suboptimal arms and increasing regret. 
\texttt{Dri-IMED} consistently achieves lower regret and tighter interquartile 
ranges than \texttt{Dri-MED} across all ablation settings, suggesting it is 
more sample-efficient under tighter constraints.

\begin{figure}[htbp]
    \centering
    \includegraphics[width=\textwidth]{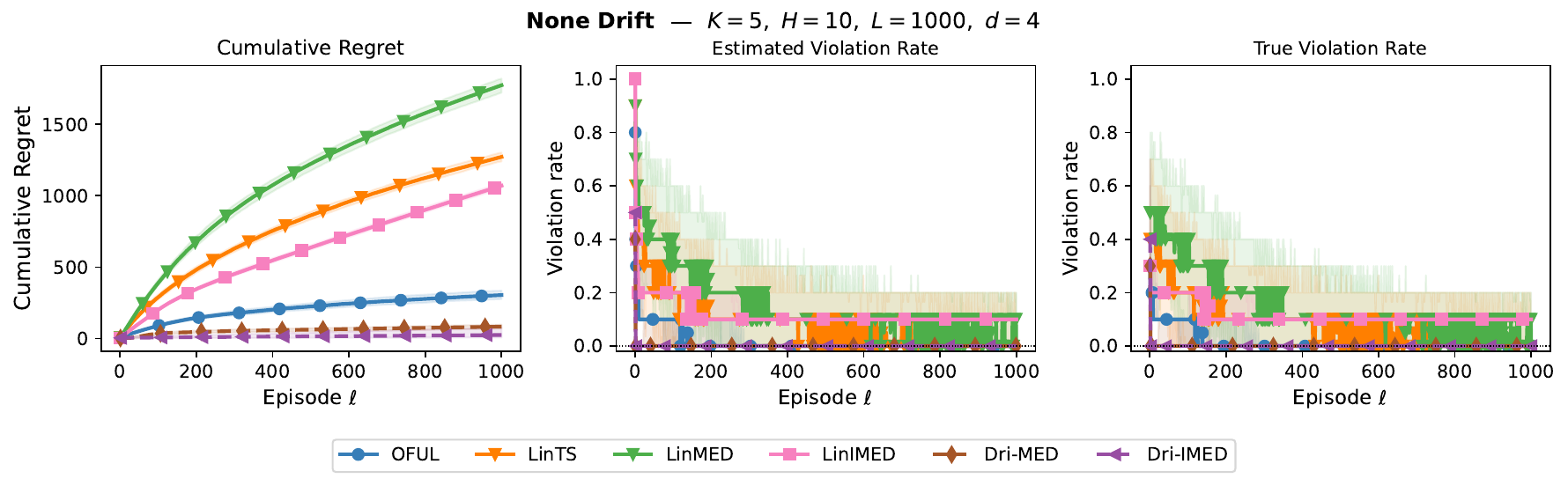}\\[4pt]
    \includegraphics[width=\textwidth]{figures/results_gradual.pdf}\\[4pt]
    \includegraphics[width=\textwidth]{figures/results_periodic.pdf}\\[4pt]
    \includegraphics[width=\textwidth]{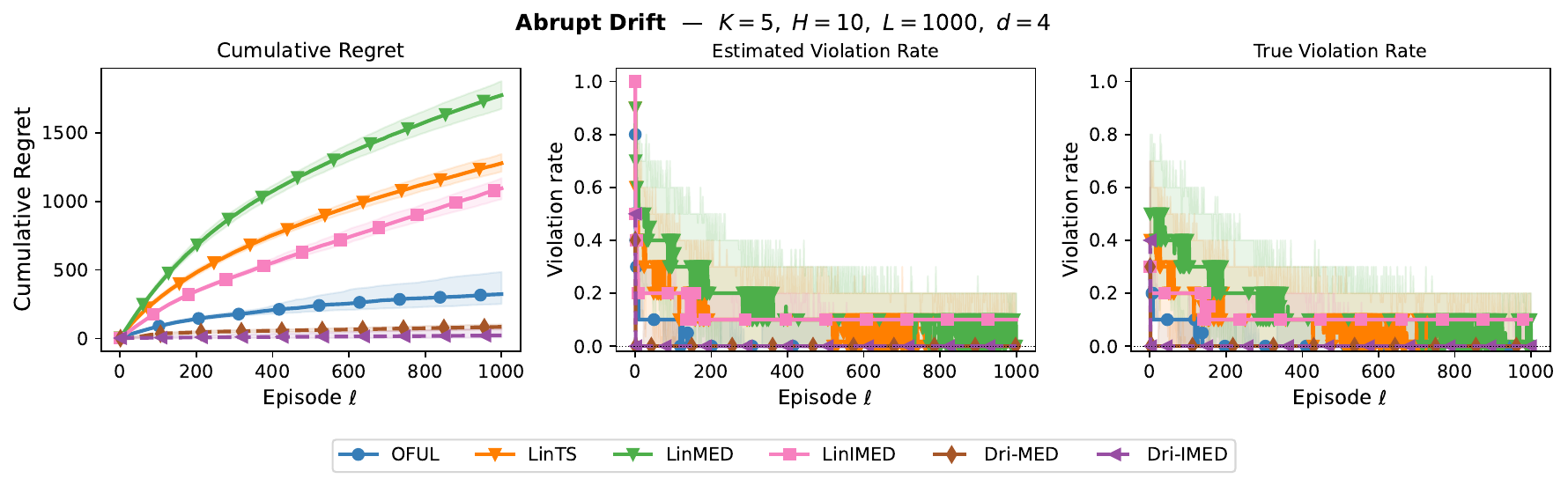}
    \caption{Cumulative regret (left), estimated violation rate (center), and true
    violation rate (right) for each drift regime. Shaded bands show the 5--95\%
    quantile range over 128 seeds.}
    \label{fig:results}
\end{figure}

\begin{figure}[htbp]
    \centering
    \includegraphics[width=\textwidth]{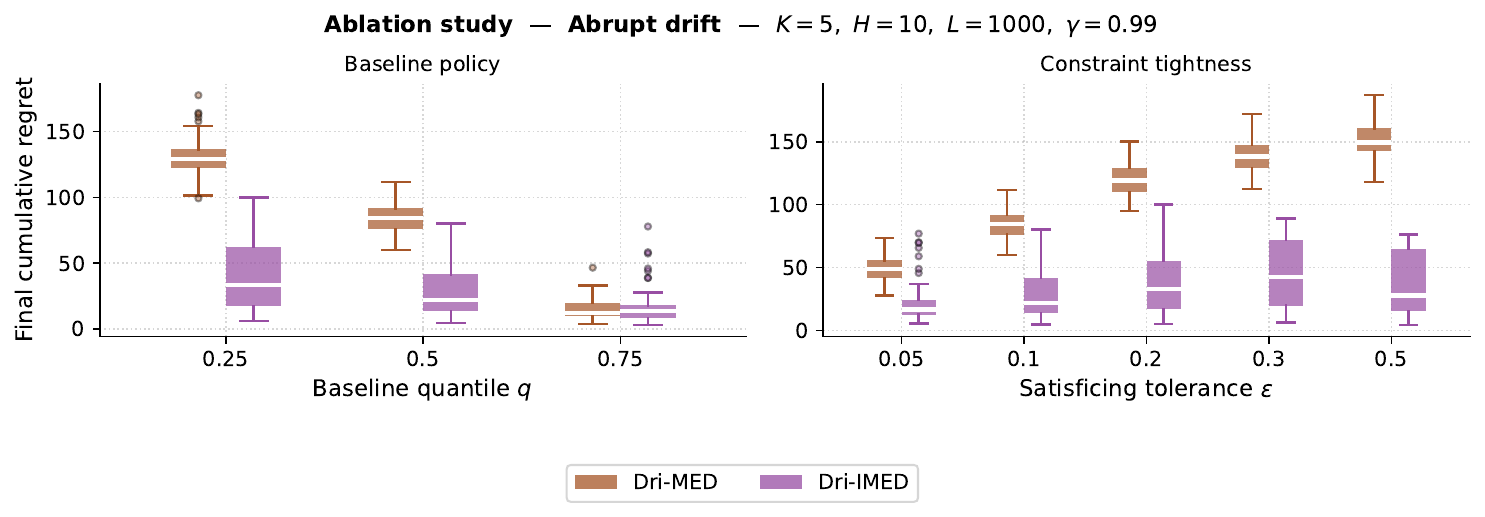}
    \caption{Ablation study on abrupt drift: effect of the baseline policy
    quantile $q$ (left) and the satisficing tolerance $\varepsilon$ (right)
    on the final cumulative regret of \texttt{Dri-MED} and \texttt{Dri-IMED}.}
    \label{fig:ablation}
\end{figure}

\begin{figure}[htbp]
    \centering
    \includegraphics[width=0.7\textwidth]{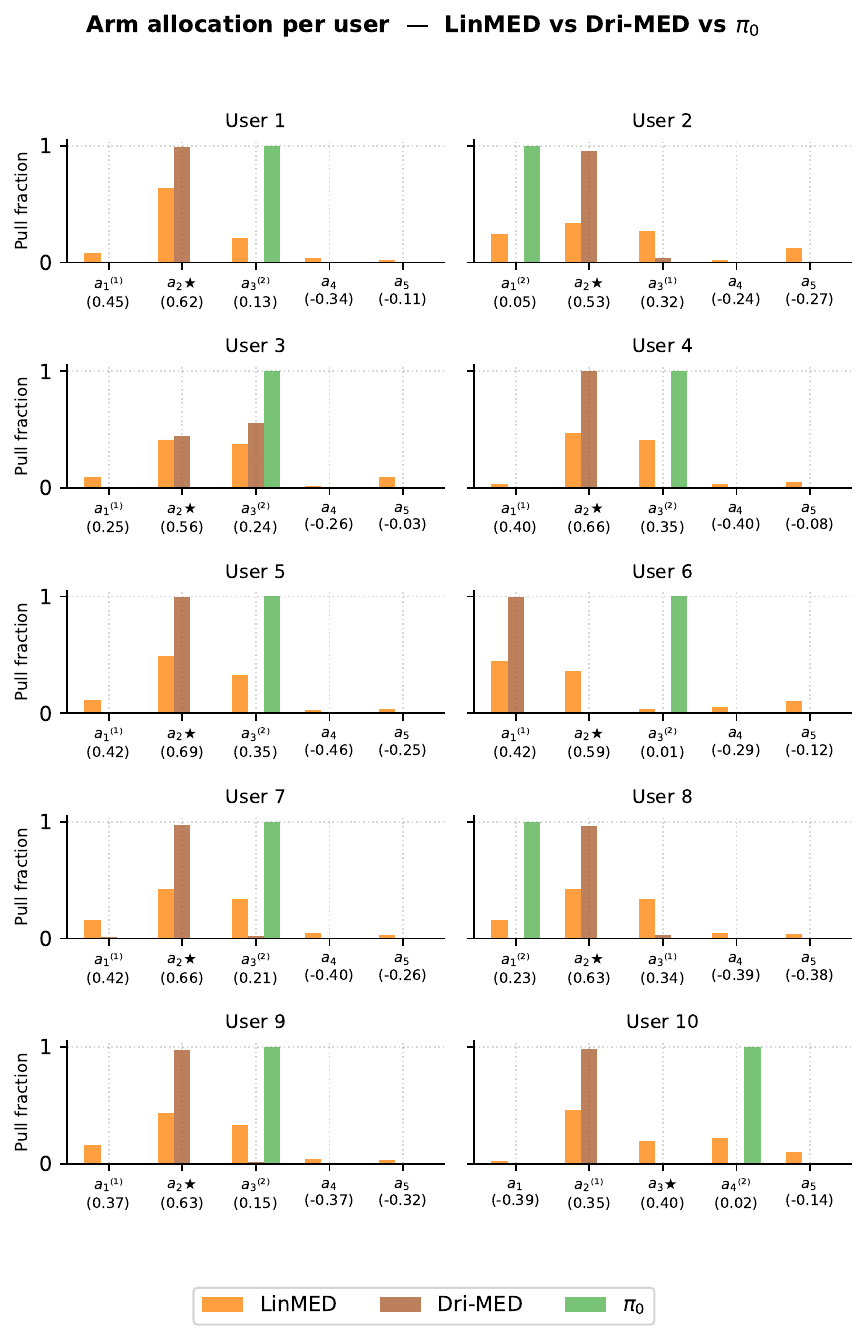}
    \caption{Arm pull fraction per user over 1000 episodes on abrupt drift.
    $\star$ denotes the oracle best arm; ${}^{(1)}$ and ${}^{(2)}$ the first
    and second suboptimal arms. Mean rewards are shown in parentheses.}
    \label{fig:arm_allocation}
\end{figure}

\newpage

\section{{\color{red}ApproxDesign}}\label{supp;app_design}

\begin{algorithm}[h!]
    \caption{\color{red}ApproxDesign}
    \label{alg:approx_design}
\begin{algorithmic}[1]
    \REQUIRE Scaled arm matrix $\mathcal{A}^{(h)} = \{\tilde\phi_a\}_{a=1}^{|\mathcal{A}|}$
    \STATE Project onto active subspace: compute PCA of $\mathcal{A}^{(h)}$, retain eigenvectors with eigenvalue $> \epsilon_0$, obtain low-dimensional matrix $\tilde X$
    \STATE \textbf{Initialise} via volume approximation: select $2d'$ arms $\mathcal{S}_0 \subseteq \mathcal{A}$ spanning the subspace by greedily picking $\arg\max_a \langle \tilde\phi_a, b_i\rangle$ and $\arg\min_a \langle \tilde\phi_a, b_i\rangle$ along each basis direction $b_i$
    \STATE Set $V \leftarrow \sum_{a \in \mathcal{S}_0} \tilde\phi_a \tilde\phi_a^\top$
    \WHILE{$\max_{a \in \mathcal{A}}\|\tilde\phi_a\|^2_{V^{-1}} > 1$}
        \STATE $a^\star \leftarrow \arg\max_{a \in \mathcal{A}} \|\tilde\phi_a\|^2_{V^{-1}}$
        \STATE $V \leftarrow V + \tilde\phi_{a^\star}\tilde\phi_{a^\star}^\top$,\quad $\mathcal{S}_0 \leftarrow \mathcal{S}_0 \cup \{a^\star\}$
    \ENDWHILE
    \STATE $q_h^{\mathrm{opt}}(a) \leftarrow |\{i : \mathcal{S}_0[i] = a\}| \;/\; |\mathcal{S}_0|$
    \RETURN $q_h^{\mathrm{opt}}$
\end{algorithmic}
\end{algorithm}

By the virtue of design of \framework, we retain the exact guaranty of eliminating highly suboptimal arms as same as~\citep{balagopalan2024minimum}. By the virtue of design of~\framework, we retain every guaranty on arm saturation by scaling both $\alpha_{\rm emp}$ and $\alpha_{\rm opt}$ by $\exp(-\bnu_{h,\ell})$ per step. Intuitively, more the value of the Lagrangian dual, more unsafe is the arm. As the Lagrangian dual grows at a rate $\bigO(\sqrt{\ell h})$ (Lemma~\ref{lem:nu-bound}), it ensures a truly unsafe arm is eliminated while augmentation of the arm set. This is a novel adaptation to make the design constraint-aware.

We omit these proofs as they directly follow from~\citep{balagopalan2024minimum} and we do not want to remain repetitive.\clearpage
\section{Useful Technical Results}
\begin{lemma}[Lower bound on gap for violating rounds]\label{lem:gap-lb}
  For any constraint violating index $(h,\ell)$, $\Delta_{A_{h,\ell},h,\ell} \ge \Dpi$.
\end{lemma}
\begin{proof}
     The proof is straightforward. We decompose the gap as:

\begin{align*}
    \Delta_{A_{h,\ell},h,\ell}&= \theta^\top(\phi_{a^*_\ell,h,\ell}-\phi_{A_{h,\ell},h,\ell})\\
    &=\theta^\top(\phi_{a^*_\ell,h,\ell} -\vp_{\pi_0})
    + \underbrace{\theta^\top(\vp_{\pi_0} - \vp_{A_{h,\ell},h,\ell})}_{\ge 0} \quad \quad \text{, iff } (h,\ell) \in \cV\\
    &\geq \min_{\ell\in[L]}\theta^{\top}(\phi_{a^*_\ell,h,\ell} -\vp_{\pi_0})  =\Dpi
\end{align*}
\end{proof}

\begin{definition}[Good event at the end of $\ell$-th episode]
    \begin{align}
        \cG_2 \defn \left\{h=H, \forall  \ell\ge 1: \|\theta - \thetahat_{\ell}\|_{\graml{\ell}}^2 \leq \beta_{\ell}(\delta_{\ell})  \right\}\label{eq:good_event_l}
    \end{align}
    Note, the index $(H,\ell)$ denotes end of the episode $\ell$. Thus for brevity, we remove the index $h$ for the event $\cG_2$. 
\end{definition}
\begin{lemma}[Heteroscedastic weighted $(H,\ell)$-th confidence set]\label{lem:conf}
    Following the information acquisition rule at the end of episode $\ell\in[L]$, i.e., $\graml{\ell} = \discount\graml{\ell-1}$, we define
\begin{align*}
    {\beta_{\ell}(\delta_{\ell})}^{1/2} \defn 
      \sqrt{2\log\frac{\det \graml{\ell-1}}{\det V_0}
           + 2\log\frac{1}{\delta_{\ell}}}
      + \sqrt{\lambda}\,S_* \,,
\end{align*}

where $\delta_{\ell}\in (0,1)$ is to characterised later on. Then $\Prob(\cG_2) \ge 1 - \suml\delta_{\ell}$.
\end{lemma}
\begin{proof}
     For this proof, we again apply the self-normalized martingale inequality of~\citep[Theorem 2]{abbasi2011improved} to the discounted Gram matrix $\graml{\ell}$. At round $(H,\ell)$, the accumulated noise over $H$ reward signal $\eta_{H,\ell} = \sumh \eta_{h,\ell}$ is $\sumh\sigma^2_{A_{h,\ell},h,\ell}$-sub-Gaussian and $\cF_{\ell}$-measurable. 
     The weighted process $\suml \gamma^\ell  \eta_{H,\ell}\bar{\phi_{\ell}}$ is a martingale, and the theorem yields
    the stated bound round-by-round. Thus, a union bound over all $\ell$ gives $\Prob(\cG_2) \ge
    1 - \suml\delta_{\ell}$.
\end{proof}

\begin{lemma}[Elliptical Potential Count Lemma: Lemma C.2 of~\citep{jun2024noise}]\label{lemma:epc_count}
    Let $x_1, x_2, \ldots, x_t \in \mathbb{R}^d$ be a sequence of vectors with $\left\|x_s\right\|_2 \leq 1, \forall s \in[t]$. Let $V_t=\lambda I+\sum_{s=1}^t x_s x_s^{\top}$ for some $\lambda>0$. Let $J=\left\{s \in[t]:\left\|x_s\right\|_{V_{s-1}^{-1}}^2 \geq L^2\right\}$ for some $L^2 \leq 1$. Then,
$$
|J| \leq 3 \frac{d}{L^2} \ln \left(1+\frac{2}{L^2 \lambda}\right)
$$
\end{lemma}

\begin{lemma}[Elliptical Potential Lemma: Proposition 2 of~\citep{abeille2017linear}]\label{lemma:epc}
    Let $x_1, x_2, \ldots, x_t \in \mathbb{R}^d$ be a sequence of vectors with $\left\|x_s\right\|_2 \leq 1, \forall s \in[t]$. Let $V_t=\lambda I+\sum_{s=1}^t x_s x_s^{\top}$ for some $\lambda>0$. Then,
$$
\sum_{s=1}^t\left\|x_s\right\|_{V_s^{-1}}^2 \leq 2 d \log \left(1+\frac{t}{d \lambda}\right) .
$$
\end{lemma}
\begin{corollary}[Elliptical Potential Lemma for rescaled features]\label{cor:epc_rescaled}
For $\|\phi_{A_{h,\ell},h,\ell}\|2^2 \leq 1,\forall h,\ell ge 1$,
\begin{align*}
    \|\phi_{A_{h,\ell},h,\ell}\|_{(\gram{h}{\ell})^{-1}}^2 \leq 2 d \log \left(1+\frac{t}{d \lambda}\right)
\end{align*}

\end{corollary}

\begin{lemma}[OFUL confidence bound lemma adapted from Theorem 2 of~\citep{abbasi2011improved}]
    Assume $\forall s \in[t], \quad\left\|a_s\right\| \leq 1$, and $\left\|\theta^*\right\|_2 \leq S$, for some fixed $\mathrm{S}>0$. We also assume $\Delta_a:=\max _{a^{\prime} \in \mathcal{A}_t}\left\langle a, \theta^*\right\rangle-\left\langle a, \theta^*\right\rangle \leq$ 1, $\forall a \in \mathcal{A}$
$$
\forall t \geq 1, \quad \mathbb{P}\left(\left\|\hat{\theta}_{t-1}-\theta^*\right\|_{V_{t-1}} \leq \sqrt{\beta_{t-1}\left(\delta_{t-1}\right)}\right) \geq 1-\delta .
$$

\end{lemma}

\begin{lemma}[Adapted from Lemma 3 in~\cite{balagopalan2024minimum}]\label{lem:prob_lb}
    Let $\tilde{f}_{h}(a), q_h$ are defined as per Algorithm~\ref{alg:med_step} where $A_{h,\ell} \neq a$ . Then $\forall h>1$,
$$
1 \geq \sum_{b \in B_h} q_h(b) f_t(b) \geq \alpha_{\mathrm{emp}} e^{-\nu_{h,\ell}}
$$
\end{lemma}

\begin{lemma}[Adapted from Lemma 5 in~\citep{balagopalan2024minimum}]
    For \framework, we have--
$$
\left\|a_{h,\ell}^*\right\|_{(V\left(p_h\right)^{\discount})^{-1}}^2 \leq \frac{2}{\alpha_{\mathrm{opt}}} \exp \left(\frac{\left\|\hat{\theta}_{h,\ell}-\theta\right\|_{\gram{h}{\ell}}^2}{\beta_{h,\ell}\left(\alpha_{h,\ell}\right)}\right) \cdot C_{\mathrm{opt}} \cdot d \log (d)
$$
where $a_{h,\ell}^*$ is true best arm at user $h$ in episode $\ell$.
\end{lemma}

\begin{lemma}[\citep{balagopalan2024minimum}]
    $$\left\|a_{h,\ell}^*\right\|_{(V\left(p_h\right)^{\discount})^{-1}}^2 \leq \frac{2 e}{\alpha_{\mathrm{opt}}} \cdot C_{\mathrm{opt}} \cdot d \log (d)$$
\end{lemma}
\end{document}